\documentclass[JIP]{apris}

\usepackage{graphicx}
\usepackage{latexsym}
\usepackage{cite}
\usepackage{amsmath,amssymb,amsfonts,amsthm}
\usepackage{subcaption}
\usepackage{algorithmic}
\usepackage{array}
\usepackage{tabularx}
\usepackage{textcomp}
\usepackage{xcolor}
\usepackage[capitalize]{cleveref}

\usepackage{tikz}
\usepackage{lipsum}
\makeatletter
\def\ipsjcopyright{
    \footnotesize © 2025 IPSJ. Personal use of this material is permitted.\newline
    DOI: XXX
}
\makeatother
\AddToHook{shipout/firstpage}{
    \begin{tikzpicture}[remember picture,overlay]
        \node[anchor=south west,xshift=1.0cm,yshift=0.8cm] at (current page.south west){\parbox{\linewidth}{\raggedright\ipsjcopyright}};
    \end{tikzpicture}
}

\theoremstyle{plain}
\newtheorem{thm}{Theorem}
\newtheorem{lem}[thm]{Lemma}
\newtheorem{crl}[thm]{Corollary}
\newtheorem{dfn}[thm]{Definition}

\crefname{equation}{Eq.}{Eqs.}
\crefname{thm}{Theorem}{Theorems}
\crefname{lem}{Lemma}{Lemmas}
\crefname{crl}{Corollary}{Corollaries}
\crefname{def}{Definition}{Definitions}

\makeatletter
\def\@fig@maybe@boldnum#1{%
  \begingroup
    \let\bf@or@normal\normalfont
    \@ifundefined{used@#1}{%
      \expandafter\gdef\csname used@#1\endcsname{used}%
      \let\bf@or@normal\bfseries
    }{}%
    \bf@or@normal\ref{#1}%
  \endgroup
}
\@ifundefined{if@prefixbold}{\newif\if@prefixbold}{}
\def\@doFigsRange#1#2#3{%
  \begingroup
    \@prefixboldfalse
    \@ifundefined{used@#2}{\@prefixboldtrue}{}%
    \@ifundefined{used@#3}{\@prefixboldtrue}{}%
    \if@prefixbold {\bfseries#1}\else #1\fi
    \@fig@maybe@boldnum{#2}--\@fig@maybe@boldnum{#3}%
  \endgroup
}
\def\figslist#1{%
  \begingroup
    \@prefixboldfalse
    \newcount\@n \@n=0
    \@for\@t:=#1\do{%
      \advance\@n by 1\relax
      \@ifundefined{used@\@t}{\@prefixboldtrue}{}%
    }%
    \if@prefixbold {\bfseries\short@figs}\else \short@figs\fi
    \newcount\@i \@i=0
    \@for\@t:=#1\do{%
      \advance\@i by 1\relax
      \@fig@maybe@boldnum{\@t}%
      \ifnum\@i<\@n
        \ifnum\@i=\numexpr\@n-1\relax
          \ifnum\@n=2\relax
            \space and\space 
          \else
            ,\space and\space 
          \fi
        \else
          ,\space 
        \fi
      \fi
    }%
  \endgroup
}
\def\Figslist#1{%
  \begingroup
    \@prefixboldfalse
    \newcount\@n \@n=0
    \@for\@t:=#1\do{%
      \advance\@n by 1\relax
      \@ifundefined{used@\@t}{\@prefixboldtrue}{}%
    }%
    \if@prefixbold {\bfseries\long@figs}\else \long@figs\fi
    \newcount\@i \@i=0
    \@for\@t:=#1\do{%
      \advance\@i by 1\relax
      \@fig@maybe@boldnum{\@t}%
      \ifnum\@i<\@n
        \ifnum\@i=\numexpr\@n-1\relax
          \ifnum\@n=2\relax
            \space and\space
          \else
            ,\space and\space
          \fi
        \else
          ,\space
        \fi
      \fi
    }%
  \endgroup
}
\def\figsrange{\@ifstar\figsrange@star\figsrange@nostar}
\def\figsrange@nostar#1#2{%
  \@doFigsRange{\short@figs}{#1}{#2}%
}
\def\figsrange@star#1#2#3{%
  \begingroup
    \@for\@t:=#3\do{%
      \@ifundefined{used@\@t}{\expandafter\gdef\csname used@\@t\endcsname{used}}{}%
    }%
  \endgroup
  \@doFigsRange{\short@figs}{#1}{#2}%
}
\def\Figsrange{\@ifstar\Figsrange@star\Figsrange@nostar}
\def\Figsrange@nostar#1#2{%
  \@doFigsRange{\long@figs}{#1}{#2}%
}
\def\Figsrange@star#1#2#3{%
  \begingroup
    \@for\@t:=#3\do{%
      \@ifundefined{used@\@t}{\expandafter\gdef\csname used@\@t\endcsname{used}}{}%
    }%
  \endgroup
  \@doFigsRange{\long@figs}{#1}{#2}%
}

\DeclareRobustCommand{\eqref}[1]{\textup{(\ref{#1})}}
\makeatother

\newif\ifpng
\pngfalse  

\def\Underline{\setbox0\hbox\bgroup\let\\\endUnderline}
\def\endUnderline{\vphantom{y}\egroup\smash{\underline{\box0}}\\}
\def\|{\verb|}

\usepackage[varg]{txfonts}
\makeatletter%
\input{ot1txtt.fd}
\makeatother%

\begin{document}

\title{Generalized Inequality-based Approach for Probabilistic WCET Estimation}

\affiliate{SU}{Graduate School of Science and Engineering, Saitama University}
\affiliate{AS}{Academic Association (Graduate School of Science and Engineering), Saitama University}
\affiliate{T4}{TIER IV Incorporated}

\author{Hayate Toba}{SU}[h.toba.498@ms.saitama-u.ac.jp]
\author{Atsushi Yano}{SU,T4}
\author{Takuya Azumi}{AS,T4}

\begin{abstract}
Estimating the probabilistic Worst-Case Execution Time (pWCET) is essential for ensuring the timing correctness of real-time applications, such as in robot IoT systems and autonomous driving systems.
While methods based on Extreme Value Theory (EVT) can provide tight bounds, they suffer from model uncertainty due to the need to decide where the upper tail of the distribution begins.
Conversely, inequality-based approaches avoid this issue but can yield pessimistic results for heavy-tailed distributions.
This paper proposes a method to reduce such pessimism by incorporating saturating functions (arctangent and hyperbolic tangent) into Chebyshev's inequality, which mitigates the influence of large outliers while preserving mathematical soundness.
Evaluations on synthetic and real-world data from the Autoware autonomous driving stack demonstrate that the proposed method achieves safe and tighter bounds for such distributions.

\end{abstract}


\maketitle

\section{Introduction}\label{sec:intro}

The estimation of the Worst-Case Execution Time (WCET) of programs is essential to ensure that complex robotic and IoT applications, such as autonomous vehicles, meet their timing constraints~\cite{2008_TECS}.
The growing computational demands of applications such as sensor fusion in modern robotics have greatly increased platform complexity (e.g., multi- and many-core processors), making traditional static timing analysis extremely difficult.
From the perspective of preference, soft real-time systems, in which occasional deadline misses are tolerable within a specified probability, deliberately prefer probabilistic over deterministic guarantees.
Consequently, measurement-based probabilistic WCET (pWCET) estimation, which predicts the exceedance probability of the upper-tail values of execution-time distributions, has gained attention.

Among existing approaches, pWCET estimation that applies Extreme Value Theory (EVT)~\cite{EVT,1928_fisher,1943_gnedenko,1974_balkema,1975_pickands} has attracted significant attention~\cite{2010_WCET,2012_ECRTS,2017_RTSS,2017_TODAES,2019_ACMSurv}.
When its prerequisites are satisfied and EVT is applied appropriately, it provides highly reliable pWCET estimates~\cite{2016_ECRTS,2017_ESLetters}.
However, there are cases where pWCET estimation is not possible under EVT~\cite{EVCondition}.
Moreover, EVT-based methods in practice suffer from model uncertainty.
This type of uncertainty arises from the mandatory selection of hyperparameters such as the threshold for upper-tail extraction, and no definitive rule yet exists for choosing optimal values.
To overcome this, an approach using Markov's inequality has been proposed to obtain a safe upper bound without such tuning~\cite{2022_ECRTS}. 
Despite its advantages, this method yields pessimistic results when applied to heavy-tailed distributions, which are observed in actual autonomous driving systems and thus pose a critical issue.

To address this pessimism, this paper proposes an inequality-based pWCET estimation method that incorporates saturating functions (arctangent and hyperbolic tangent) into Chebyshev's inequality.
The proposed method suppresses the influence of extreme outliers while preserving mathematical soundness, thereby providing a tighter pWCET estimation for heavy-tailed distributions without suffering from model uncertainty.
The practicality of our method is demonstrated through evaluations using data from the Autoware~\cite{autoware} autonomous driving stack.

The primary contributions of this paper are summarized as follows:
\begin{itemize}
    \item Performing pWCET estimation without model uncertainty by inequality-based approach.
    \item Reducing the pessimism of the estimation by mitigating the impact of outliers using saturating functions.
    \item Demonstrating the practicality of the proposed method through the evaluation using data from an actual system.
\end{itemize}

The remainder of this paper is organized as follows.
\cref{sec:background} presents the relevant background.
\cref{sec:proposed_method} provides a detailed explanation of the proposed method.
\cref{sec:evaluation} presents evaluations using an actual autonomous driving system and synthetic task sets.
\cref{sec:related_work} introduces related research.
\cref{sec:conclusion} presents the conclusions and future work.

\section{Background and Problem Definition}\label{sec:background}

This section aims to provide the necessary background information and clarify the problems that need to be addressed.

\subsection{Probabilistic WCET}\label{ssub:bg_pwcet}

The execution time of a program is not fixed but varies with each run.
The goal of pWCET estimation is to find a reliable upper bound for this variation.
An intuitive visualization of this concept is illustrated in \figref{fig:pwcet}.
The y-axis represents the exceedance probability (as a complementary CDF or CCDF), i.e., the probability that the execution time on the x-axis is exceeded.
The black dashed lines show different execution time distributions obtained from multiple measurement campaigns.
The objective is to find a single curve that safely bounds all these empirical distributions.
The pWCET curve, shown in red, serves as this upper bound.
For any target exceedance probability $p$, the curve provides a corresponding execution time $x$ that is guaranteed to be a safe estimate.

\begin{figure}[tb]
    \centering
    \includegraphics[width=\linewidth]{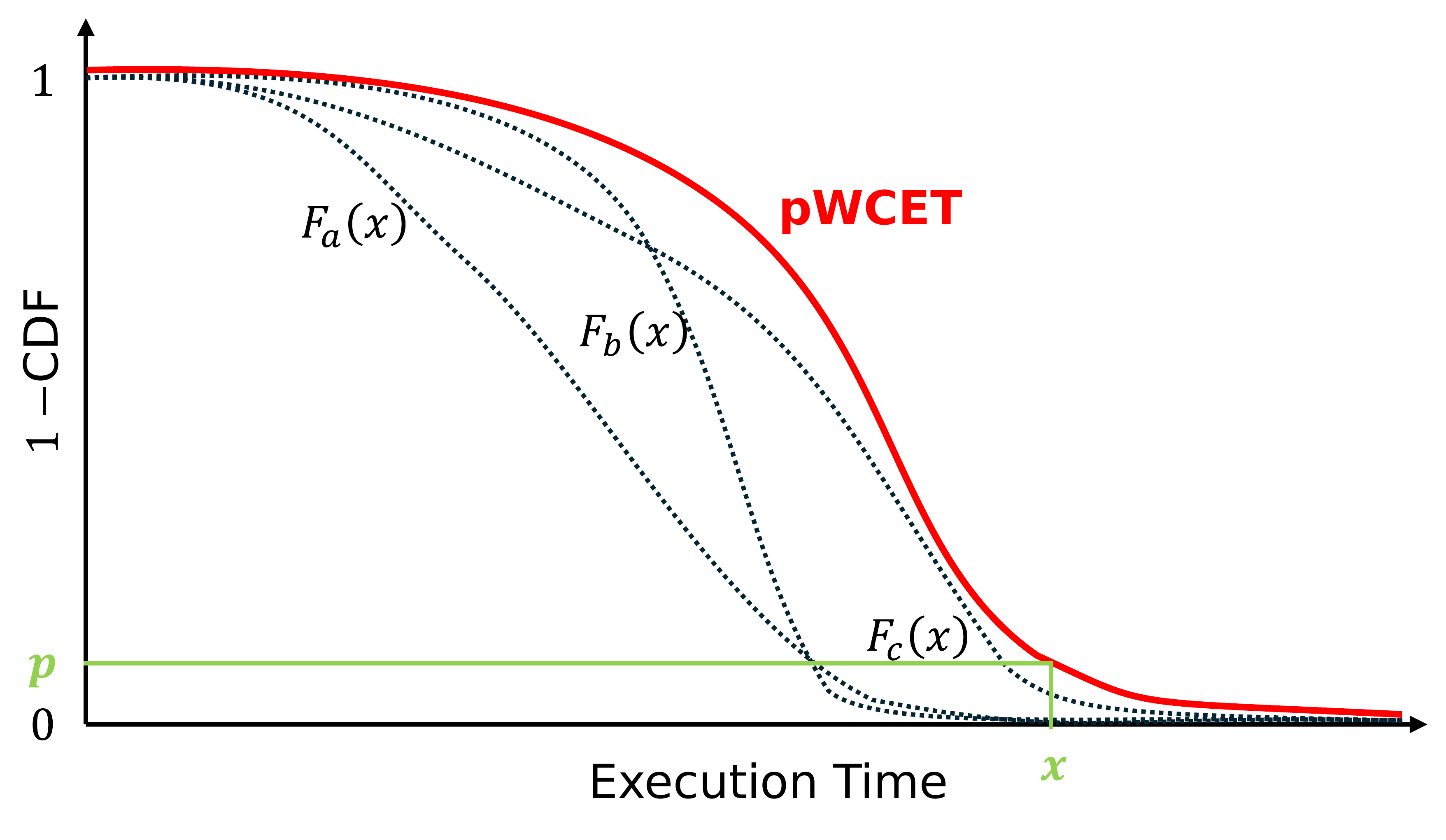}
    \caption{Execution times and pWCET}
    \label{fig:pwcet}
\end{figure}

\subsection{Chebyshev and Markov Inequalities}\label{ssub:bg_inequalities}

Our inequality-based approach is founded on Chebyshev's inequality~\cite{1867_chebyshev}, which provides a safe upper bound without the model uncertainty of EVT.

\begin{dfn}\label{dfn:exp}
    Let $X$ be a discrete random variable with probability mass function $f_{X}$.
    The expected value of $X$ is given by:
    \begin{equation}\label{eq:exp}
        E(X)=\sum_{x\in X}x\cdot f_{X}(x)=\sum_{x\in X}x\cdot P(X=x).
    \end{equation}
\end{dfn}

\begin{thm}[Chebyshev's Inequality~\cite{1867_chebyshev}]\label{thm:chebyshev}
    For a non-negative random variable $X$, $b > 0$, and a non-negative, monotonically increasing function $f$, the following holds:
    \begin{equation}\label{eq:chebyshev}
        P(X\ge b)\le\frac{E(f(X))}{f(b)}.
    \end{equation}
\end{thm}

\noindent When $f$ is the identity function, that is, when $f(X) = X$, this particular form is called Markov's inequality~\cite{1884_markov}.

For WCET estimation, the random variable $X$ corresponds to the execution time of the program, and $b$ corresponds to the execution time for which the exceedance probability is to be upper bounded.

A prior method for the pWCET estimation applies this inequality by using the power-of-$k$ function, $f(X)=X^{k}$, where $k>0$~\cite{2022_ECRTS}.
This yields the following bound:

\begin{dfn}\label{dfn:exp_k}
    Let $X$ be a discrete random variable and $k$ a positive real value.
    The expected value of $X^{k}$, also defined as the $k$-th moment of $X$, is given by:
    \begin{equation}\label{eq:exp_k}
        E(X^{k})=\sum_{x}x^{k}\cdot P(X=x).
    \end{equation}
\end{dfn}

\begin{crl}[Markov's Inequality to the power-of-$k$~\cite{2022_ECRTS}]\label{crl:markov_k}
    For a non-negative random variable $X$, $b > 0$, and the power-of-$k$ function $f(X)=X^{k}$, Markov's inequality to the power-of-$k$ yields:
    \begin{equation}\label{eq:markov_k}
        P(X\ge b)\le\frac{E(X^{k})}{b^{k}}.
    \end{equation}
\end{crl}

Since different values of $k$ produce different bounds, the state-of-the-art MEMIK method takes the lower envelope of all such bounds to achieve the tightest possible estimate~\cite{2022_ECRTS}:

\begin{dfn}[MEMIK bound~\cite{2022_ECRTS}]\label{dfn:memik}
    For a non-negative random variable $X$, $b > 0$, and $k > 0$, the MEMIK bound is defined as follows:
    \begin{equation}\label{eq:memik}
        P(X\ge b)\le\min_{k}\frac{E(X^{k})}{b^{k}}\quad\mathrm{for}\quad k>0.
    \end{equation}
\end{dfn}

While this approach is effective, it can produce pessimistic bounds for heavy-tailed distributions, as will be shown in \cref{ssub:motivating_example}.

\subsection{Motivating Example}\label{ssub:motivating_example}

\begin{figure}[t]
    \centering
    \begin{minipage}{\linewidth}
        \centering
        \ifpng
            \includegraphics[draft=false,width=0.85\textwidth]{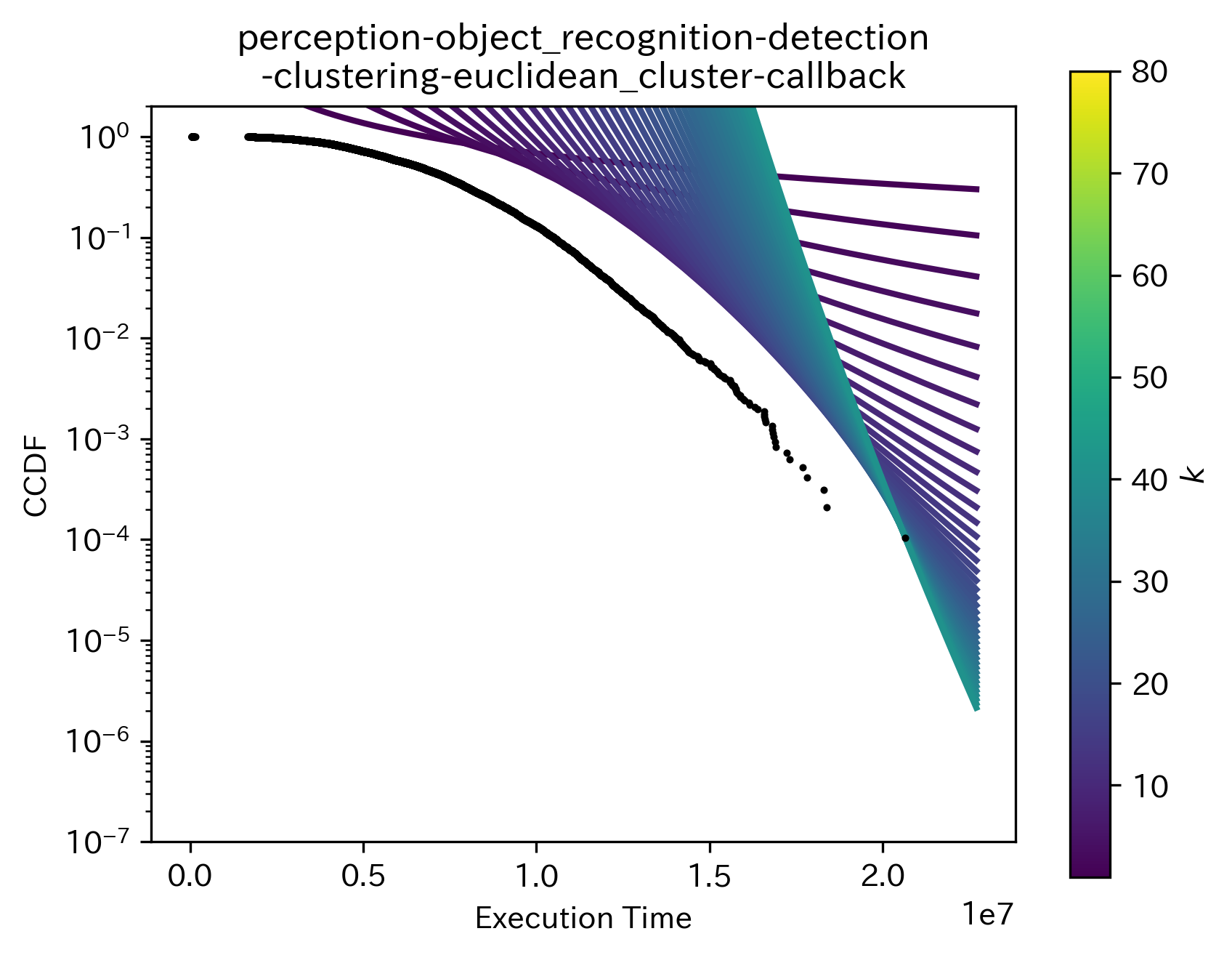}
        \else
            \includegraphics[draft=false,width=0.85\textwidth]{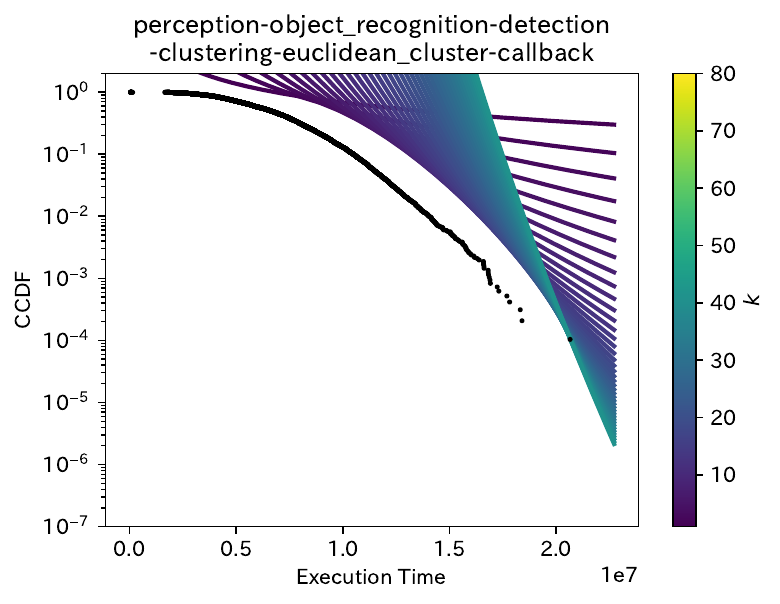}
        \fi
        \subcaption{Light-tailed distribution}
        \label{fig:mik_light}
    \end{minipage}
    \begin{minipage}{\linewidth}
        \centering
        \ifpng
            \includegraphics[draft=false,width=0.85\textwidth]{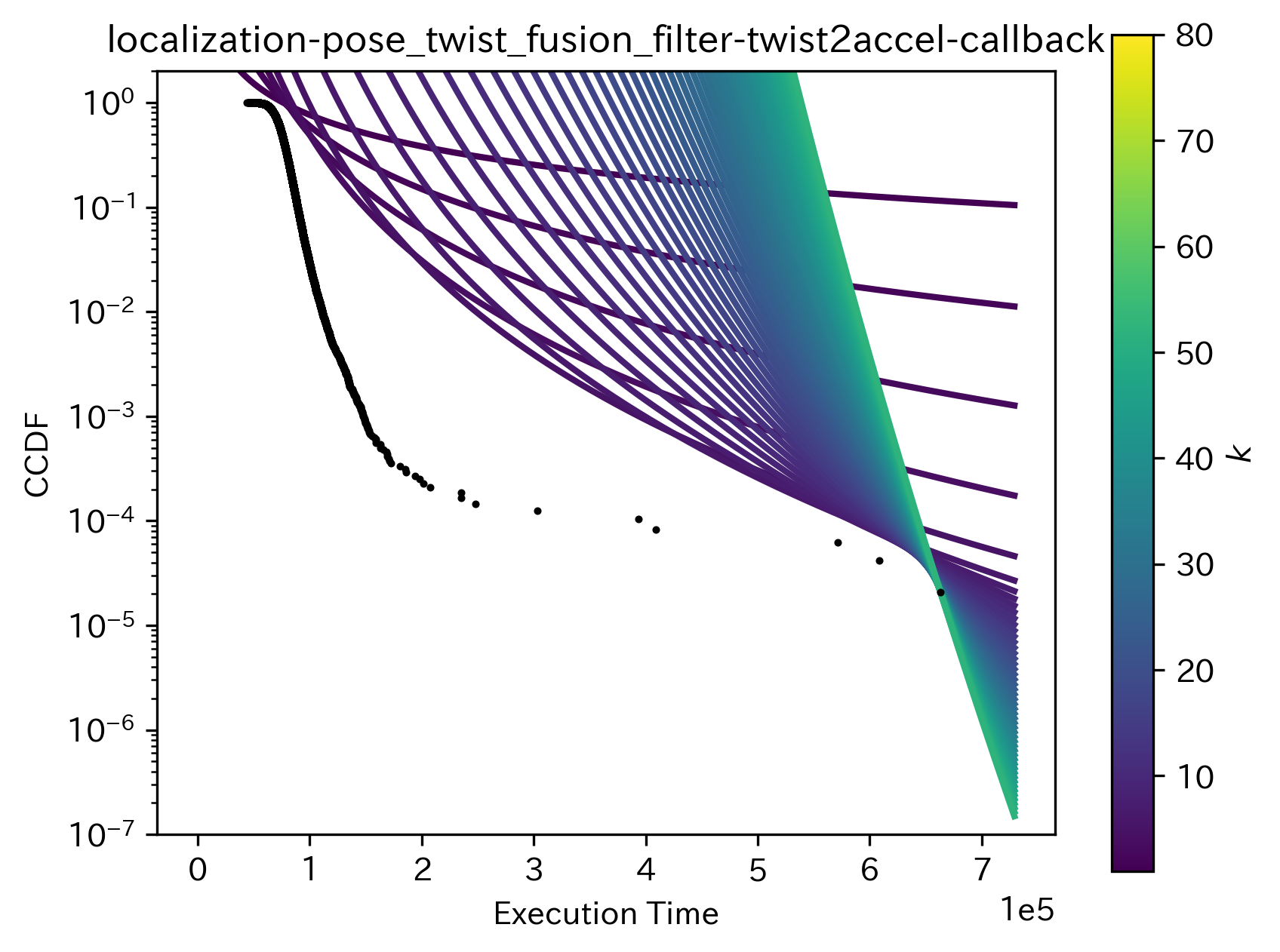}
        \else
            \includegraphics[draft=false,width=0.85\textwidth]{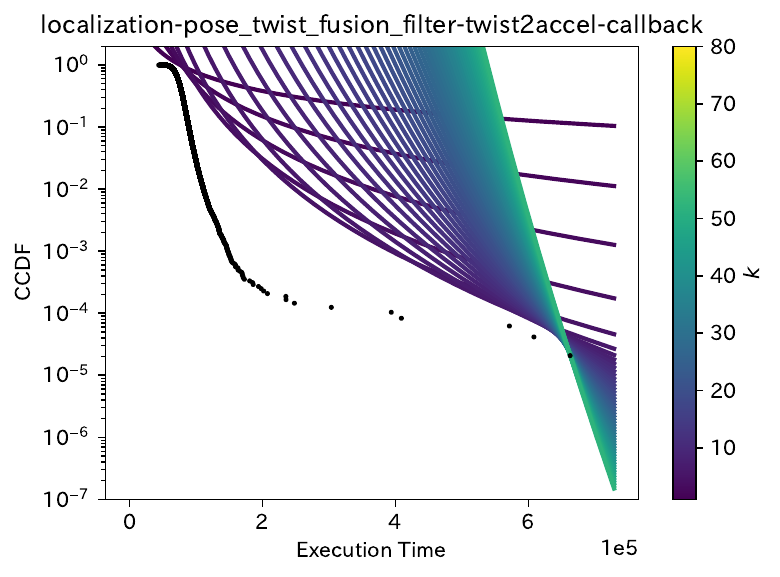}
        \fi
        \subcaption{Heavy-tailed distribution}
        \label{fig:mik_heavy}
    \end{minipage}
    \caption{Results of applying Inequality~\labelcref{eq:markov_k} to measurements from an actual autonomous driving system}
    \label{fig:mik_problem}
\end{figure}

To motivate our work, we apply the MEMIK bound~\cite{2022_ECRTS} to execution-time data from Autoware~\cite{autoware}, a real autonomous driving system.
The results are shown in \figref{fig:mik_problem}, where the black curve is the observed distribution (CCDF) and the colored curves are the bounds from Inequality~\labelcref{eq:markov_k} for different $k$ values; the MEMIK bound is their lower envelope.

For a light-tailed distribution from the LiDAR clustering callback (\figref{fig:mik_light}), the MEMIK bound is tight and effective.
However, for a heavy-tailed distribution from the vehicle acceleration estimation callback (\figref{fig:mik_heavy}), the bound is highly pessimistic.
For instance, at an exceedance probability of $10^{-3}$, the estimated WCET is approximately twice the largest observed value.
This pessimism reveals a fundamental limitation of using the power-of-$k$ function with Chebyshev's inequality for such distributions, motivating our proposal of a more suitable function.

\section{Proposed Method}\label{sec:proposed_method}

Our approach adopts Chebyshev's inequality and employs a function capable of providing tight estimations even for heavy-tailed distributions, which is illustrated in the example in \cref{ssub:motivating_example}.
To identify such a function, the first step involves analyzing the causes of significant overestimation when MEMIK is applied to heavy-tailed distributions.

The core of MEMIK lies in Inequality~\labelcref{eq:markov_k}, where $X$ represents a random variable representing execution time, and $b$ denotes an arbitrary execution time value.
Thus, the probability of the execution time exceeding a given value is bounded by $E(X^{k})/b^{k}$.
When $k$ is fixed, the exceedance probability for a specific $b$ depends solely on the numerator $E(X^{k})$.
In other words, the magnitude of $E(X^{k})$ directly determines the extent of the exceedance probability.
For distributions that include rare, extremely large values, $E(X^{k})$ becomes large, leading to higher exceedance probabilities.
This effect is expected to become more pronounced as $k$ increases.
The issue with MEMIK arises from this sensitivity to the expected value of rare extreme values, which leads to overestimation for heavy-tailed distributions.
To address this, we propose to use functions that mitigate the sensitivity of the expected value, thereby reducing the impact of rare extreme values.

A variety of candidate functions can be considered to prevent the expected value from becoming excessively large for rare extreme values.
Among the promising options, two functions are proposed: $f(x) = (\arctan(x/d))^{k}$ and $f(x) = (\tanh(x/d))^{k}$.
$\arctan$ and $\tanh$ are saturating functions that grow slowly as $x$ increases, allowing the expected value to remain less sensitive to rare large observations.
This helps reduce the pessimism of the resulting bound when applied to heavy-tailed distributions.
Furthermore, both functions are well-known and computationally simple to evaluate, making them practical choices for real-time analysis.
Notably, sigmoid-shaped functions often used to limit the range of values can be represented by or approximated with $\tanh$, and are thus covered within the proposed formulation.
The following sections first demonstrate that these functions can be utilized in Chebyshev's inequality, and then explain their properties. 
According to \cref{thm:chebyshev}, the properties that the function $f$ must satisfy to be used in Chebyshev's inequality are non-negativity and monotonicity.
After verifying these properties, the bounds derived from these functions are formalized.

\begin{figure}[t]
    \centering
    \begin{minipage}{0.78\linewidth}
        \hspace{-18pt}
        \centering
        \includegraphics[draft=false,width=\textwidth]{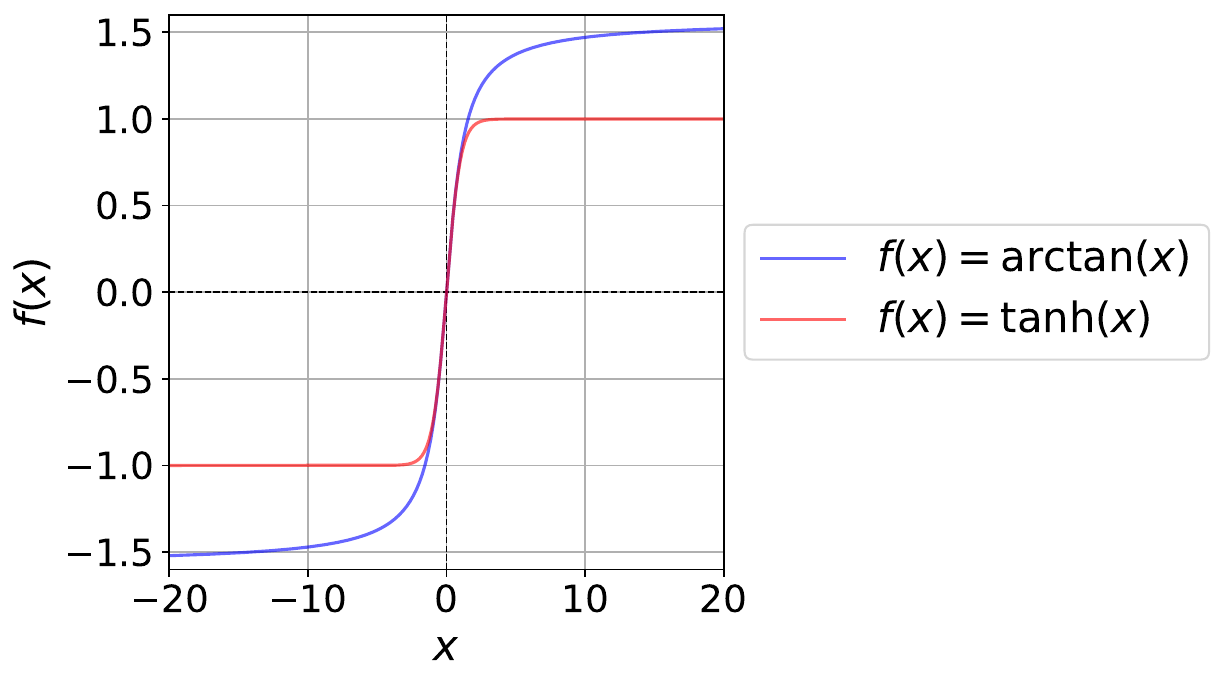}
        \subcaption{$f(x)$ around the origin}
        \label{fig:func_origin}
    \end{minipage}\\
    \begin{minipage}{0.84\linewidth}
        \centering
        \includegraphics[draft=false,width=\textwidth]{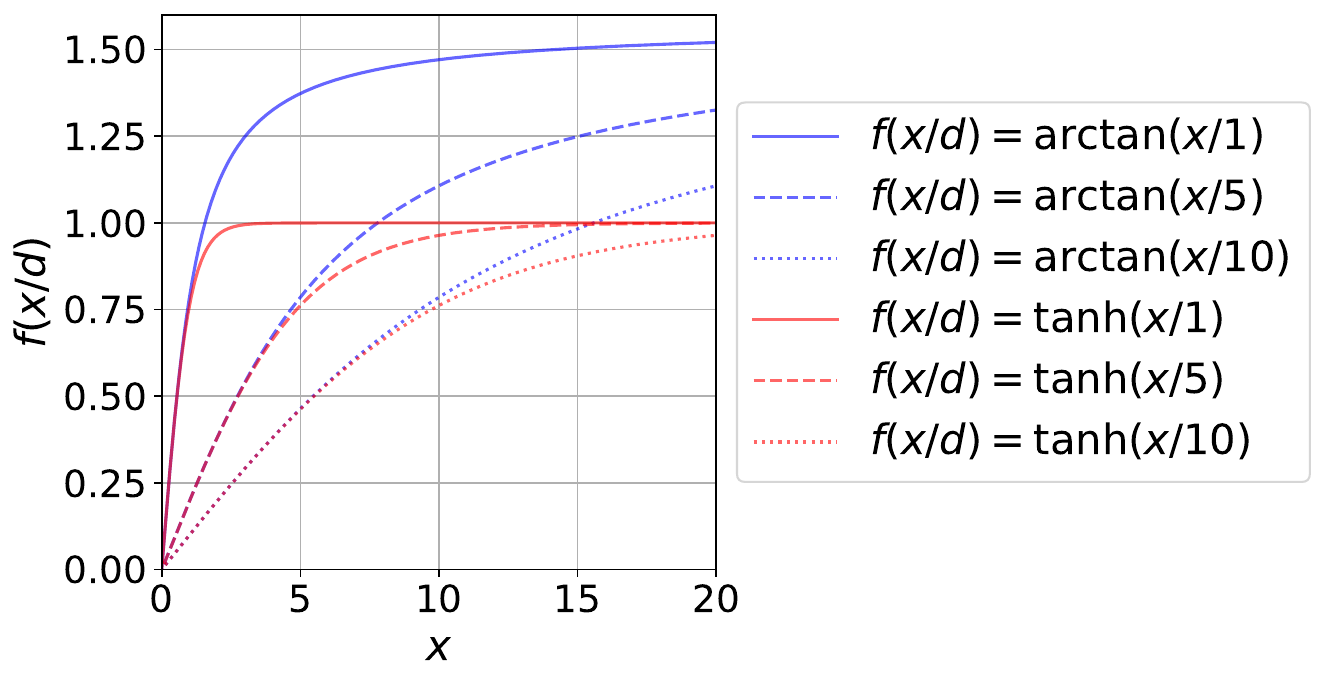}
        \subcaption{$f(x/d)$ for $x>0$}
        \label{fig:func_positive}
    \end{minipage}
    \caption{Shape of the functions $f(x)=\arctan(x)$ and $f(x)=\tanh(x)$}
    \label{fig:func}
\end{figure}

The first function is defined as $f(x) = (\arctan(x/d))^{k}$, where $d$ and $k$ are positive real numbers.
The base function $f(x) = \arctan(x)$ is illustrated in \figref{fig:func_origin} with a blue line.
However, since execution times are always positive, the focus is on $x\in\mathbb{R}^{+}=\{x\in\mathbb{R}\mid x>0\}$, as shown in \figref{fig:func_positive}. 
As evident from the figure, the growth of the function gradually slows down, which is desirable for mitigating excessive increases in the expected value caused by extreme values.
However, directly using $f(x)=\arctan(x)$ is impractical, because the range of execution times and the output range of the function do not align well.
To address this issue, a positive real number $d$ is introduced as a divisor. 
The resulting function, $f(x)=(\arctan(x/d))^{k}$, is both non-negative and monotonically increasing.

\begin{lem}\label{lem:arctan_ng_mono}
    The function $f(x)=(\arctan(x/d))^{k}$ satisfies the following properties, where $d,k\in\mathbb{R}^{+}$: 
    \begin{align}
        &\forall x\in\mathbb{R}^{+},\quad f(x)>0\label{prop:non_negative} \\
        \mathrm{and}\quad&\forall a,b\in\mathbb{R}^{+},\quad a<b\Rightarrow f(a)\le f(b).\label{prop:mono_increase}
    \end{align}
\end{lem}

\begin{proof}
    First, the validity of Property~\labelcref{prop:mono_increase} is demonstrated.
    The derivative of $f(x)$ is given as:
    \begin{align*}
        \frac{\mathrm{d}}{\mathrm{d}x}\left(\arctan\left(\frac{x}{d}\right)\right)^{k}=k\cdot\left(\arctan\left(\frac{x}{d}\right)\right)^{k-1}\cdot\frac{1}{1+\left(\frac{x}{d}\right)^{2}}\cdot\frac{1}{d}.
    \end{align*}
    Since $x,d,k\in\mathbb{R^{+}}$, all terms are positive, and the following holds:
    \begin{align*}
        \frac{\mathrm{d}}{\mathrm{d}x}f(x)>0.
    \end{align*}
    Thus, Property~\labelcref{prop:mono_increase} is satisfied, and $f(x)$ is a monotonically increasing function.
    Additionally, $f(0)=(\arctan(0/d))^{k}=0$, confirming that Property~\labelcref{prop:non_negative} is also satisfied, making $f(x)$ a non-negative function.
\end{proof}

As $f(x)=(\arctan(x/d))^{k}$ is a non-negative and monotonically increasing function, it can be incorporated into Chebyshev's inequality.

\begin{dfn}\label{dfn:exp_arctan}
    Let $X$ be a discrete random variable, and $d$ and $k$ be positive real values.
    The expected value of $(\arctan(X/d))^{k}$ is given by:
    \begin{align}
        E\left(\left(\arctan\frac{X}{d}\right)^{k}\right)=\sum_{x}\left(\arctan\frac{x}{d}\right)^{k}\cdot P(X=x).\label{eq:exp_arctan}
    \end{align}
\end{dfn}

\begin{crl}\label{crl:chebyshev_arctan}
    For a non-negative random variable $X$, $b>0$, and the function $f(X)=(\arctan(X/d))^{k}$, the following inequality yields:
    \begin{align}
        P(X\ge b)\le\frac{E\left(\left(\arctan\frac{X}{d}\right)^{k}\right)}{\left(\arctan\frac{b}{d}\right)^{k}}.\label{eq:chebyshev_arctan}
    \end{align}
\end{crl}

\begin{proof}
    This follows from \cref{thm:chebyshev} and \cref{lem:arctan_ng_mono}.
\end{proof}

Thus, the pWCET bound is defined.

\begin{dfn}\label{dfn:memik_arctan}
    For a non-negative random variable $X$, $b>0$, and $d,k>0$, the pWCET bound using the $\arctan$ function is defined as follows:
    \begin{align}
        P(X\ge b)\le\min_{d,k}\frac{E\left(\left(\arctan\frac{X}{d}\right)^{k}\right)}{\left(\arctan\frac{b}{d}\right)^{k}}\quad\mathrm{for}\quad d,k>0.\label{eq:memik_arctan}
    \end{align}
\end{dfn}

The second function is defined as $f(x)=(\tanh(x/d))^{k}$, where $d$ and $k$ are positive real numbers.
The base function $f(x) = \tanh(x)$ is illustrated in \figref{fig:func_origin} with a red line.
Since execution times are always positive, the focus is on $x \in \mathbb{R}^{+}$, as shown in \figref{fig:func_positive}.
Similar to the first function, the growth of this function also gradually slows down, but it exhibits distinct characteristics.
Here, $f(x)=(\tanh(x/d))^{k}$ is a non-negative and monotonically increasing function.

\begin{lem}\label{lem:tanh_ng_mono}
    The function $f(x)=(\tanh(x/d))^{k}$ satisfies Properties \labelcref{prop:non_negative,prop:mono_increase}, where $d,k\in\mathbb{R}^{+}$.
\end{lem}

\begin{proof}
    The proof is similar to that of \cref{lem:arctan_ng_mono} and is omitted for brevity.
\end{proof}

Since $f(x)=(\tanh(x/d))^{k}$ is a non-negative and monotonically increasing function, it can be incorporated into Chebyshev's inequality.

\begin{dfn}\label{dfn:exp_tanh}
    Let $X$ be a discrete random variable, and $d$ and $k$ be positive real values.
    The expected value of $(\tanh(X/d))^{k}$ is given by:
    \begin{align}
        E\left(\left(\tanh\frac{X}{d}\right)^{k}\right)=\sum_{x}\left(\tanh\frac{x}{d}\right)^{k}\cdot P(X=x).\label{eq:exp_tanh}
    \end{align}
\end{dfn}

\begin{crl}\label{crl:chebyshev_tanh}
    For a non-negative random variable $X$, $b>0$, and the function $f(X)=(\tanh(X/d))^{k}$, the following inequality yields:
    \begin{align}
        P(X\ge b)\le\frac{E\left(\left(\tanh\frac{X}{d}\right)^{k}\right)}{\left(\tanh\frac{b}{d}\right)^{k}}.\label{eq:chebyshev_tanh}
    \end{align}
\end{crl}

\begin{proof}
    This follows from \cref{thm:chebyshev} and \cref{lem:tanh_ng_mono}.
\end{proof}

Therefore, the pWCET bound is defined.

\begin{dfn}\label{dfn:memik_tanh}
    For a non-negative random variable $X$, $b>0$, and $d,k>0$, the pWCET bound using the $\tanh$ function is defined as follows:
    \begin{align}
        P(X\ge b)\le\min_{d,k}\frac{E\left(\left(\tanh\frac{X}{d}\right)^{k}\right)}{\left(\tanh\frac{b}{d}\right)^{k}}\quad\mathrm{for}\quad d,k>0.\label{eq:memik_tanh}
    \end{align}
\end{dfn}

As is evident from \figref{fig:func_positive}, both functions have an approximately linear interval.
If the parameters allow the execution time to fall within this interval, the proposed functions behave similarly to $f(x)=x^k$.
Thus, even for light-tailed distributions, the proposed method is capable of constructing tight bounds.
Additionally, since both $\arctan(x/d)^k$ and $\tanh(x/d)^k$ are non-negative and monotonically increasing, the use of these functions in Chebyshev’s inequality guarantees that the derived bounds are theoretically safe.

\section{Evaluation}\label{sec:evaluation}

The effectiveness of the proposed method is evaluated in this section through two complementary experiments: one using synthetic data and the other using real-world execution time data collected from an autonomous driving system, Autoware~\cite{autoware}.
The synthetic data experiment is designed to assess the general applicability of the method across a wide range of probability distributions, including those with heavy tails or multimodal characteristics.
This controlled setting allows us to compare the tightness and safety of our method against existing approaches under known ground truth.
Following this, we apply the method to real measurement data, which exhibits characteristics different from idealized synthetic distributions, to examine its practicality in actual embedded real-time systems.
Throughout the evaluation, a comparative analysis with existing methods is included to highlight the advantages and limitations of each.

\subsection{Synthetic Data}\label{ssub:eval_synthetic}

\begin{table}[tb]
    \centering
    \caption{Distributions used for the evaluation and their respective parameters}
    \label{tab:synthetic_params}
    \renewcommand{\arraystretch}{1.1}
    \begin{tabular}{l|ll}\hline\hline
        \textbf{Acronym} & \textbf{Type} & \textbf{Parameters}\\ \hline
        GaussianA & Gaussian & $\mu=100, \sigma=10$ \\
        GaussianB & Gaussian & $\mu=100, \sigma=50$ \\ \hline
        WeibullA & Weibull & $\alpha=4, \lambda=80$ \\
        WeibullB & Weibull & $\alpha=8, \lambda=80$ \\ \hline
        BetaA & Beta & $\alpha=8, \beta=0.25$ \\
        BetaB & Beta & $\alpha=8, \beta=0.125$ \\ \hline
        GammaA & Gamma & $\alpha=100, \lambda=1$ \\
        GammaB & Gamma & $\alpha=150, \lambda=1$ \\ \hline
        MixtureA & Mixture of Gaussians & $\mu=\{5, 50, 100\}, \sigma=10,$ \\
        & & $w=\{0.6, 0.39, 0.01\}$ \\
        MixtureB & Mixture of Gaussians & $\mu=\{50, 100, 400\}, \sigma=50,$ \\
        & & $w=\{0.6, 0.39, 0.01\}$ \\ \hline
        MixtureC & Mixture of Weibulls & $\lambda=\{5, 50, 100\}, \alpha=4,$ \\
        & & $w=\{0.6, 0.39, 0.01\}$ \\
        MixtureD & Mixture of Weibulls & $\lambda=\{5, 50, 100\}, \alpha=8,$ \\
        & & $w=\{0.6, 0.39, 0.01\}$ \\ \hline
    \end{tabular}
    \vspace{-3mm}
\end{table}

The general effectiveness of the proposed method is examined through experiments using synthetic data.
Following prior work, the target distributions include Gaussian, Weibull, Beta, Gamma, and mixture distributions of them, covering a broad range of tail behaviors and distributional shapes.
The evaluation metric is tightness, defined as the ratio between the estimated execution time at a specified exceedance probability and the corresponding true quantile of the underlying distribution.
That is, tightness values greater than $1.0$ and close to $1.0$ are desirable.
A tightness value below $1.0$ indicates an underestimation.
This metric reflects the degree of conservativeness in the estimation while ensuring safety.

The types and parameters of the distributions used in the evaluation are summarized in \tabref{tab:synthetic_params}.
For mixture distributions, the parameter $w$ represents the weight of each component.
The parameter settings are based on those used in prior work.
For all non-mixture distributions, the cumulative distribution function (CDF) has a closed-form inverse, allowing the quantiles to be computed analytically; these values are treated as ground truth.
In contrast, mixture distributions lack a closed-form inverse for their CDFs, and thus quantiles cannot be computed directly.
Therefore, numerical root-finding using the bisection method is employed to determine the quantiles, which are then used as the ground truth.
The sample size is fixed at $n = 10^6$.

Three methods are compared:
\begin{itemize}
    \item \textbf{MEMIK}~\cite{2022_ECRTS}: a prior method based on Markov's inequality
    \item \textbf{ATAN}: the proposed method using the $\arctan$ function
    \item \textbf{TANH}: the proposed method using the $\tanh$ function
\end{itemize}
In Ref.~\cite{2022_ECRTS}, MEMIK is shown to outperform other methods, including those based on EVT.
For each method, WCET is estimated at target exceedance probabilities of $10^{-7},\ 10^{-8},\ 10^{-9},\ 10^{-10},\ 10^{-11},\ 10^{-12},\ 10^{-13},\ 10^{-14},\ 10^{-15}$, and the corresponding tightness is evaluated.
To ensure fairness, all methods use the same sample set.

The estimation results for each distribution are shown in \figsrange*{fig:eval1_gaussianA}{fig:eval1_mixtureD}{fig:eval1_gaussianB,fig:eval1_weibullA,fig:eval1_weibullB,fig:eval1_betaA,fig:eval1_betaB,fig:eval1_gammaA,fig:eval1_gammaB,fig:eval1_mixtureA,fig:eval1_mixtureB,fig:eval1_mixtureC}.
The results for the Gaussian distributions are presented in \figslist{fig:eval1_gaussianA,fig:eval1_gaussianB}, and it can be observed that the proposed methods improve the estimates under both parameter settings.
Notably, the improvement becomes more pronounced as the target probability decreases.
One interesting observation is that, for GaussianA, ATAN and TANH yield almost identical results, while MEMIK produces significantly larger estimates.
In contrast, for GaussianB, MEMIK and ATAN produce similar results, while TANH yields smaller estimates.
These differences suggest that the estimation results vary depending on the parameters of the distribution, even within the same distribution type, due to the different curvature characteristics of the saturating functions used in the proposed methods.

The results for the Weibull distributions are shown in \figslist{fig:eval1_weibullA,fig:eval1_weibullB}.
As with the Gaussian case, the proposed methods improve the estimates, with a particularly large margin for WeibullA.
To prevent underestimation, every method applies the restricting-$k$ algorithm~\cite{2022_ECRTS} to Inequalities \labelcref{eq:memik,eq:memik_arctan,eq:memik_tanh}, but this safeguard introduces additional pessimism at lower probabilities.
The proposed approaches alleviate this side effect and yield noticeably tighter bounds.
For WeibullB, MEMIK and ATAN produce nearly identical estimates, underscoring once again that the behavior depends on the specific parameter settings of the distribution.

\begin{figure}[tb]
    \centering
    \begin{minipage}[t]{0.45\linewidth}\vspace{0pt}
        \centering
        \includegraphics[width=\linewidth]{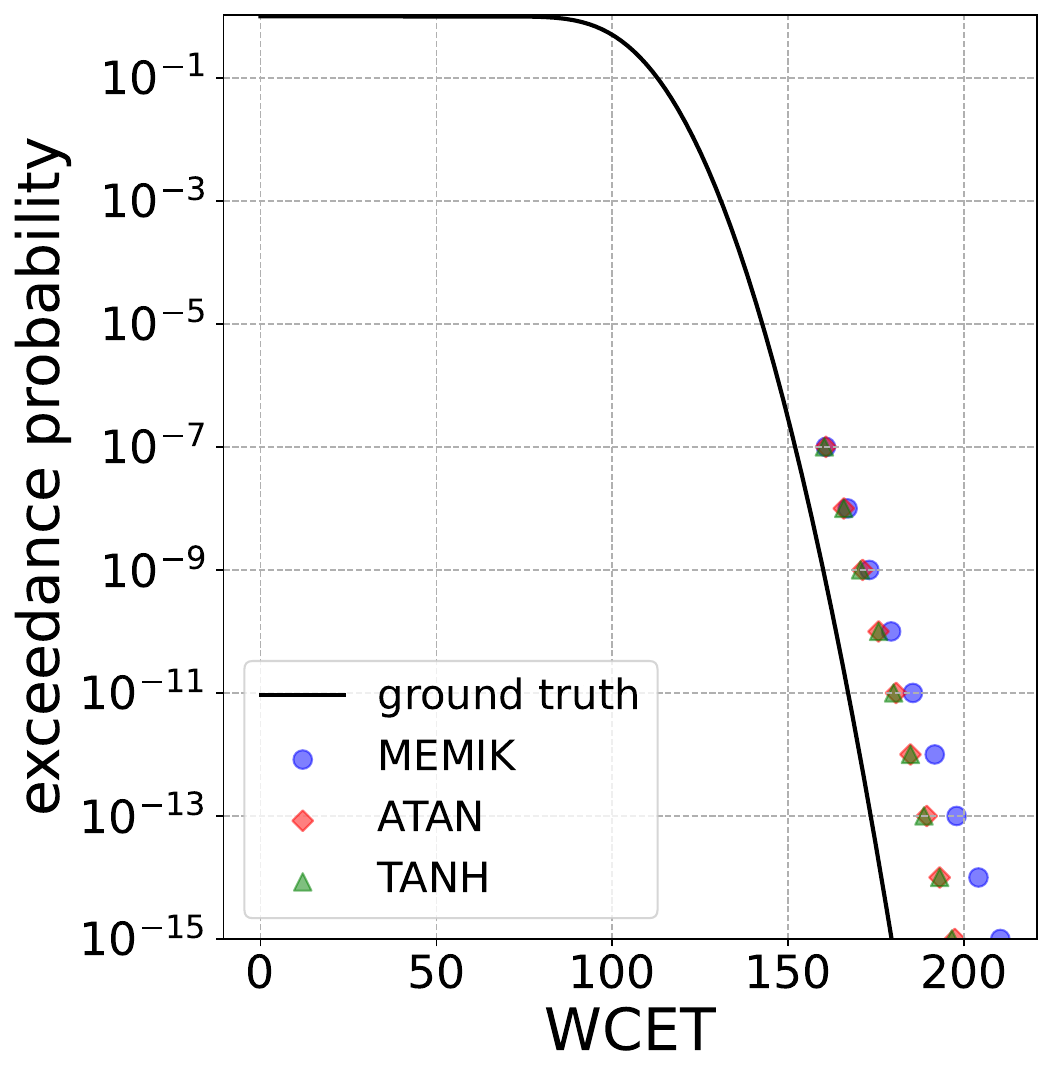}
        \subcaption{Estimated WCET values for GaussianA}
        \label{fig:eval1_gaussianA_ccdf}
    \end{minipage}\hfill
    \begin{minipage}[t]{0.45\linewidth}\vspace{0pt}
        \centering
        \includegraphics[width=\linewidth]{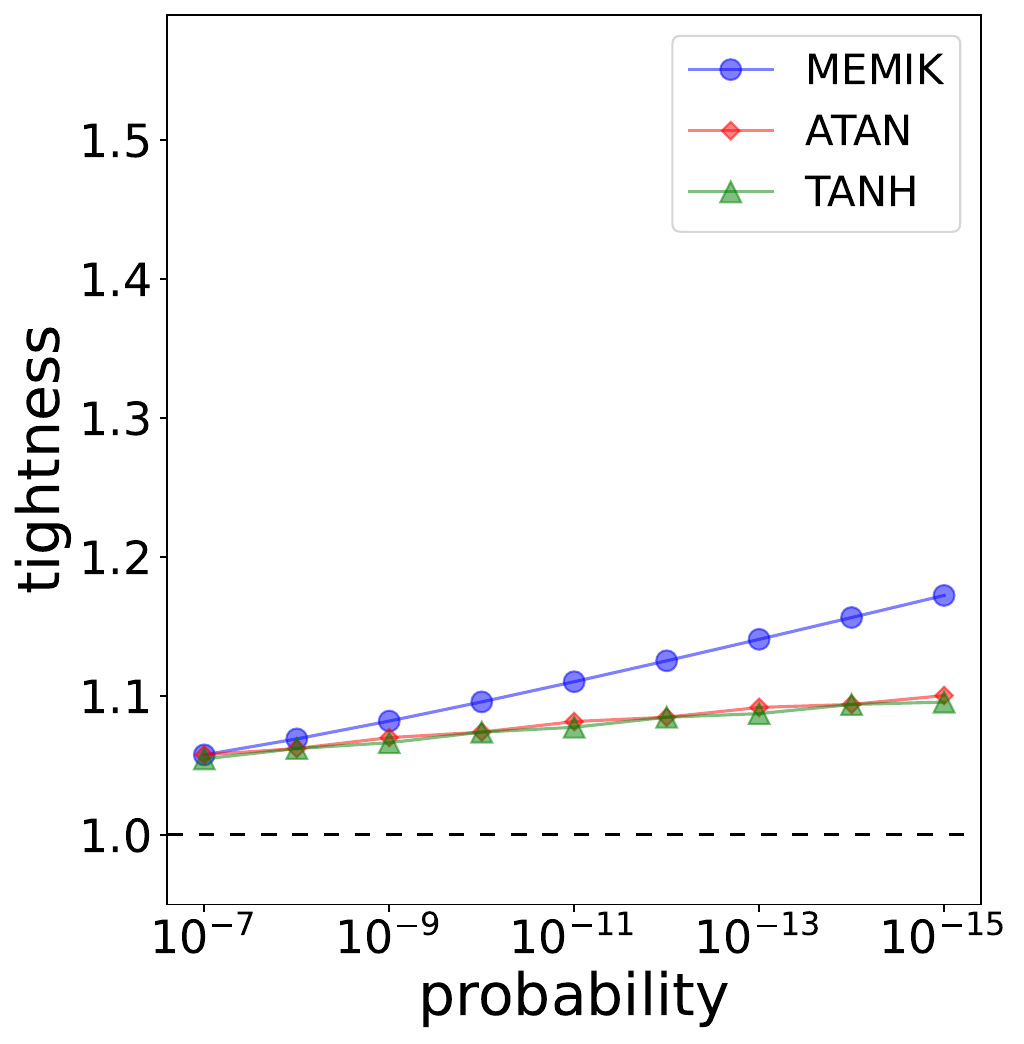}
        \subcaption{Tightness for GaussianA}
        \label{fig:eval1_gaussianA_tightness}
    \end{minipage}
    \caption{Evaluation result for GaussianA}
    \label{fig:eval1_gaussianA}
\end{figure}
\vspace{4mm}
\begin{figure}[tb]
    \centering
    \begin{minipage}[t]{0.45\linewidth}\vspace{0pt}
        \centering
        \includegraphics[width=\linewidth]{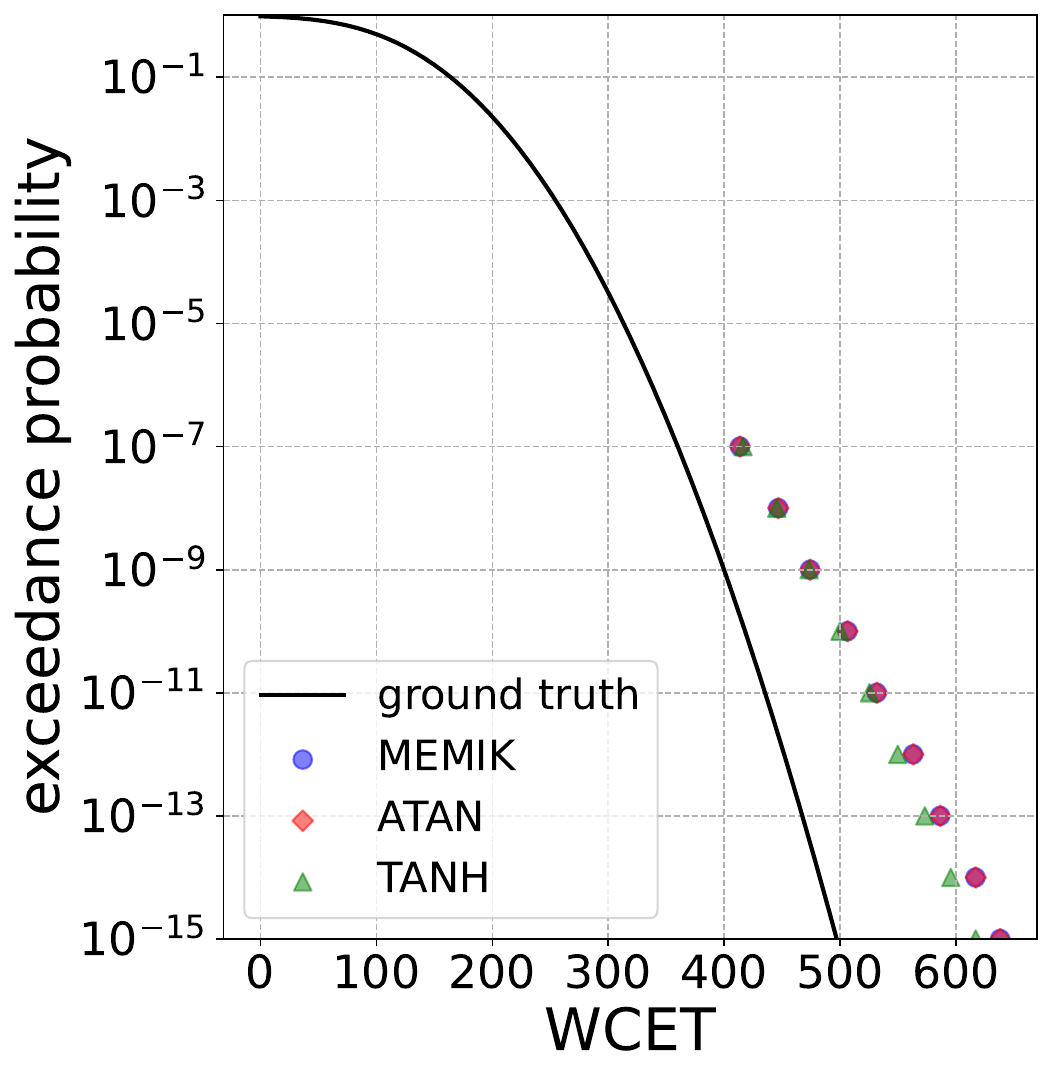}
        \subcaption{Estimated WCET values for GaussianB}
        \label{fig:eval1_gaussianB_ccdf}
    \end{minipage}\hfill
    \begin{minipage}[t]{0.45\linewidth}\vspace{0pt}
        \centering
        \includegraphics[width=\linewidth]{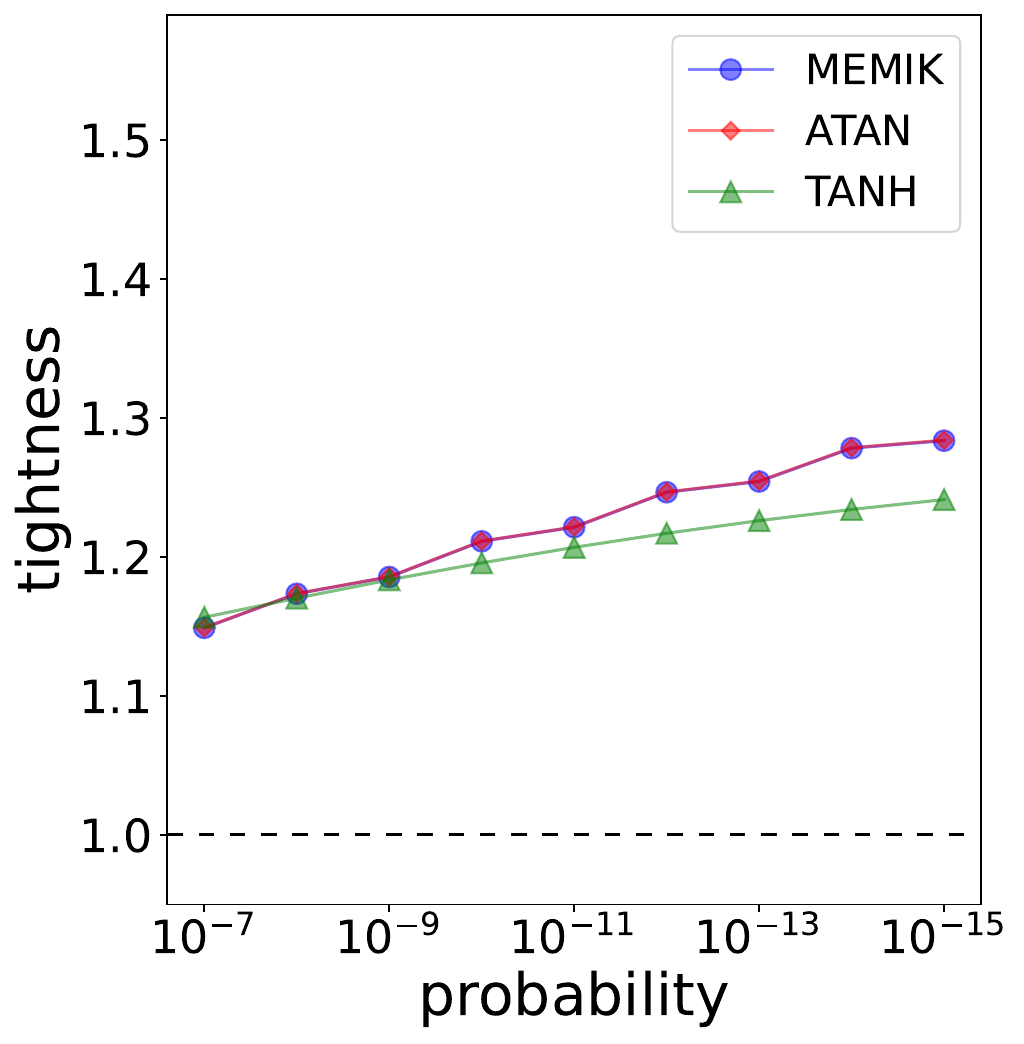}
        \subcaption{Tightness for GaussianB}
        \label{fig:eval1_gaussianB_tightness}
    \end{minipage}
    \caption{Evaluation result for GaussianB}
    \label{fig:eval1_gaussianB}
\vspace{4mm}
\end{figure}
\begin{figure}[tb]
    \centering
    \begin{minipage}[t]{0.45\linewidth}\vspace{0pt}
        \centering
        \includegraphics[width=\linewidth]{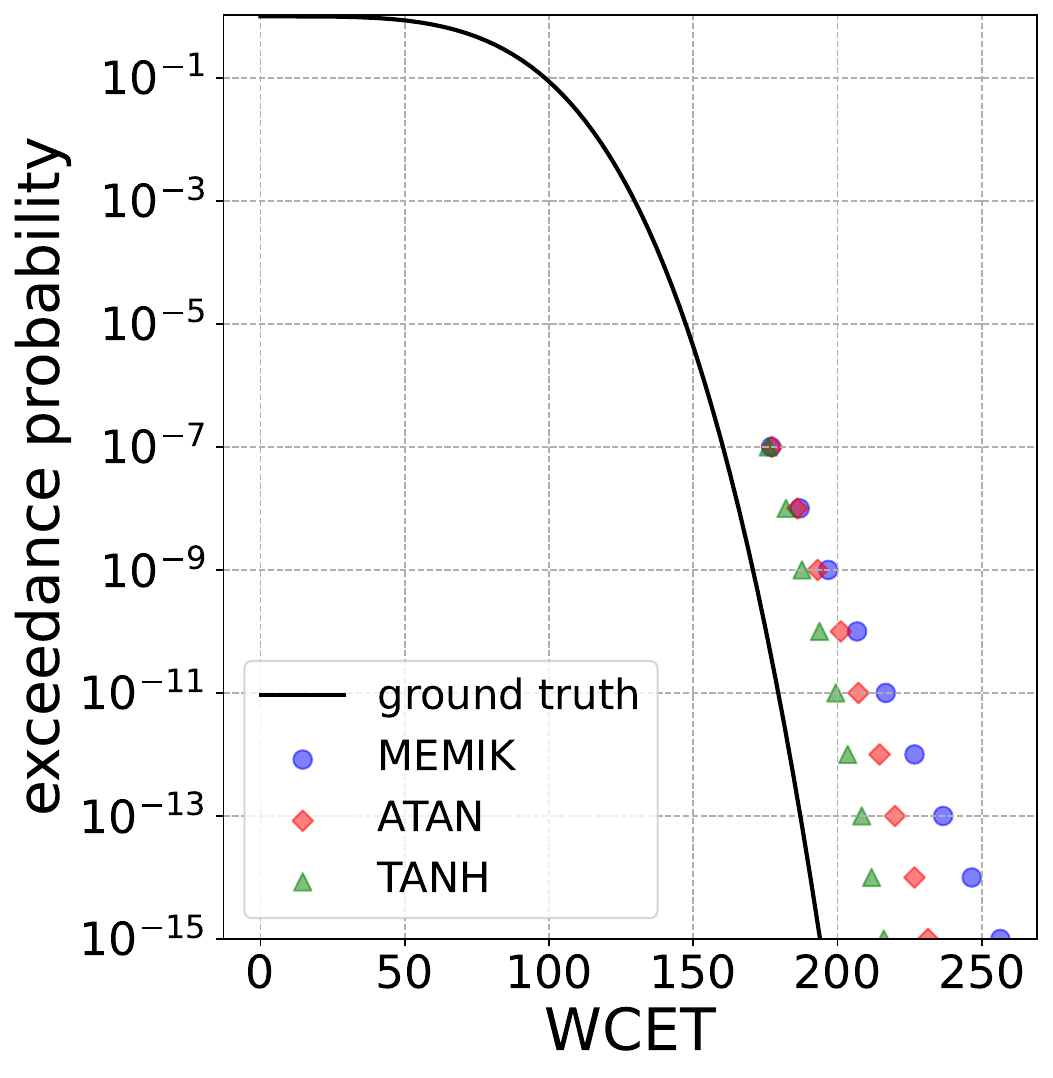}
        \subcaption{Estimated WCET values for WeibullA}
        \label{fig:eval1_weibullA_ccdf}
    \end{minipage}\hfill
    \begin{minipage}[t]{0.45\linewidth}\vspace{0pt}
        \centering
        \includegraphics[width=\linewidth]{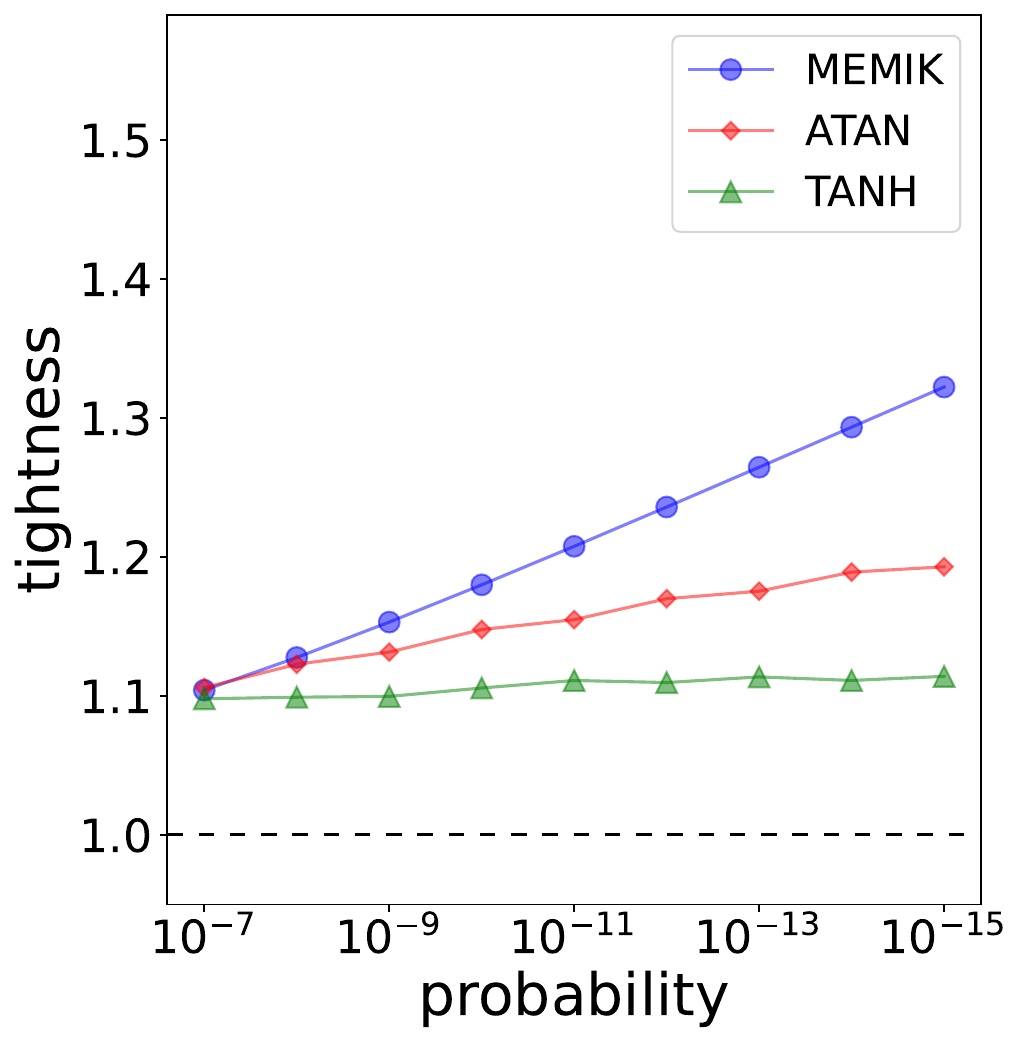}
        \subcaption{Tightness for WeibullA}
        \label{fig:eval1_weibullA_tightness}
    \end{minipage}
    \caption{Evaluation result for WeibullA}
    \label{fig:eval1_weibullA}
\vspace{4mm}
\end{figure}
\begin{figure}[tb]
    \centering
    \begin{minipage}[t]{0.45\linewidth}\vspace{0pt}
        \centering
        \includegraphics[width=\linewidth]{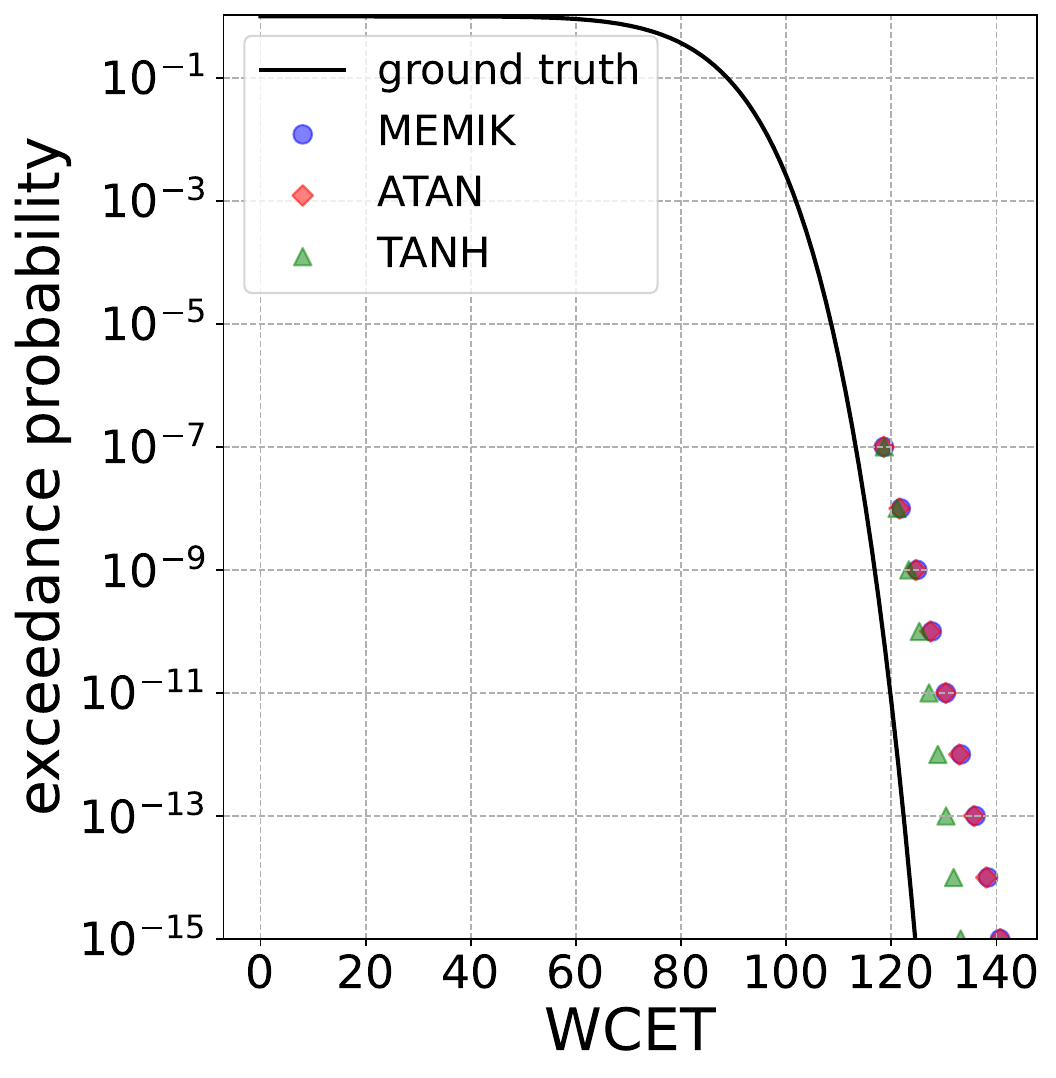}
        \subcaption{Estimated WCET values for WeibullB}
        \label{fig:eval1_weibullB_ccdf}
    \end{minipage}\hfill
    \begin{minipage}[t]{0.45\linewidth}\vspace{0pt}
        \centering
        \includegraphics[width=\linewidth]{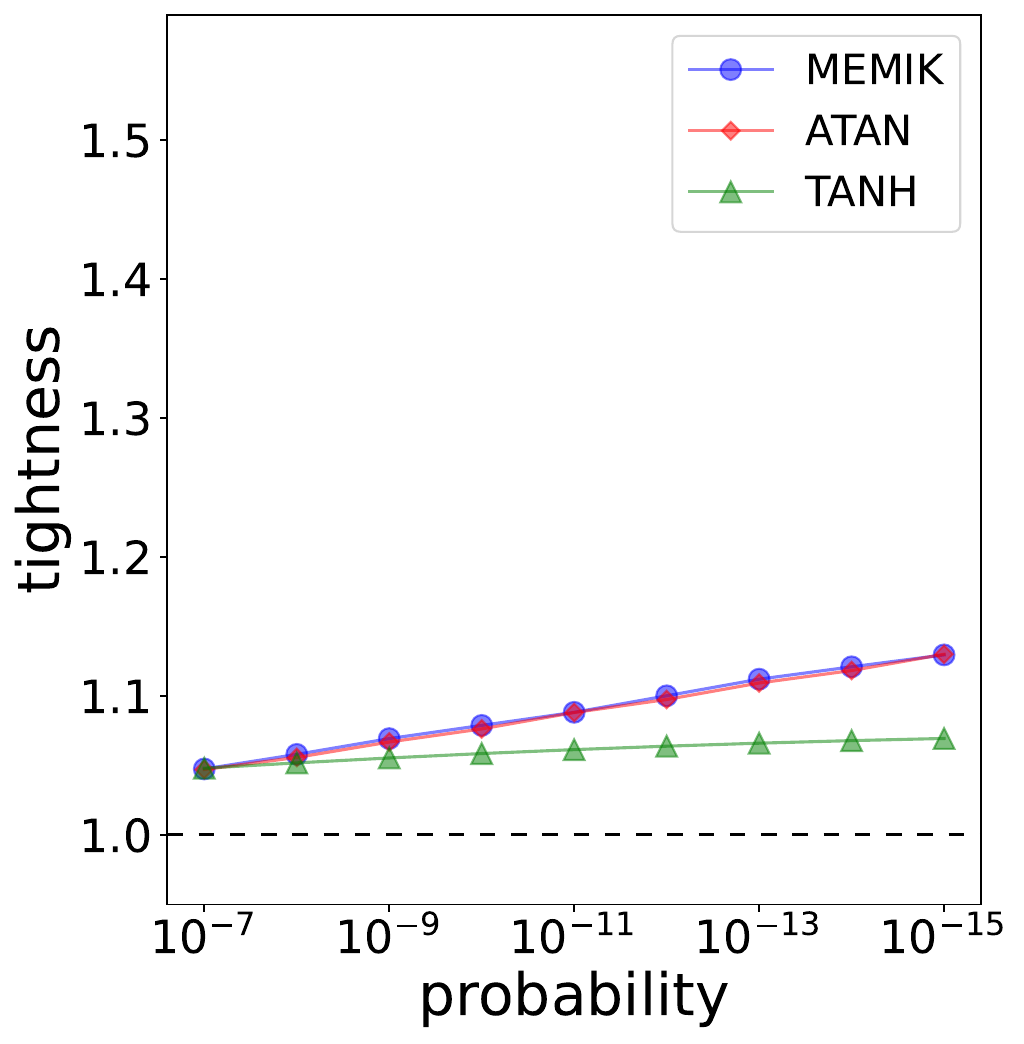}
        \subcaption{Tightness for WeibullB}
        \label{fig:eval1_weibullB_tightness}
    \end{minipage}
    \caption{Evaluation result for WeibullB}
    \label{fig:eval1_weibullB}
\vspace{4mm}
\end{figure}
\begin{figure}[tb]
    \centering
    \begin{minipage}[t]{0.45\linewidth}\vspace{0pt}
        \centering
        \includegraphics[width=\linewidth]{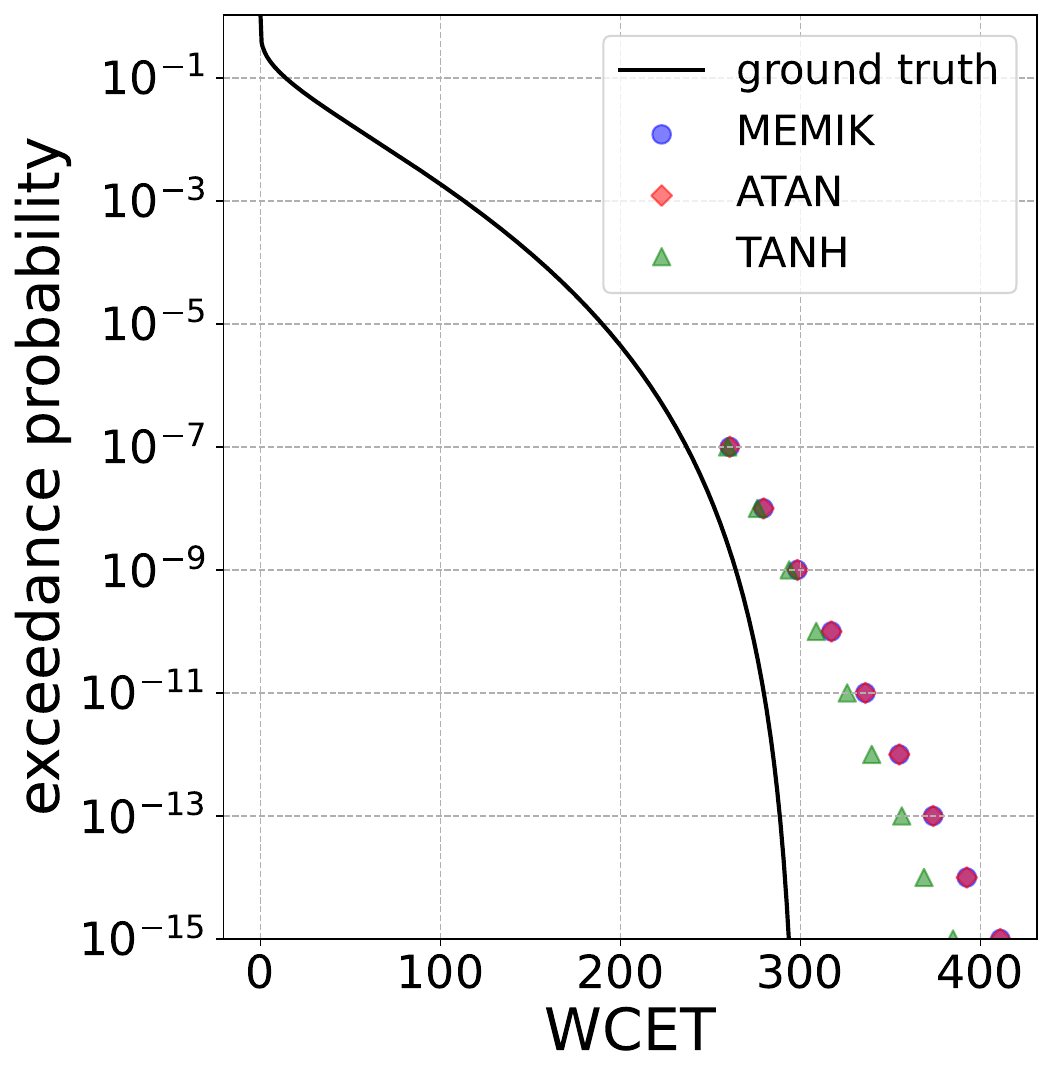}
        \subcaption{Estimated WCET values for BetaA}
        \label{fig:eval1_betaA_ccdf}
    \end{minipage}\hfill
    \begin{minipage}[t]{0.45\linewidth}\vspace{0pt}
        \centering
        \includegraphics[width=\linewidth]{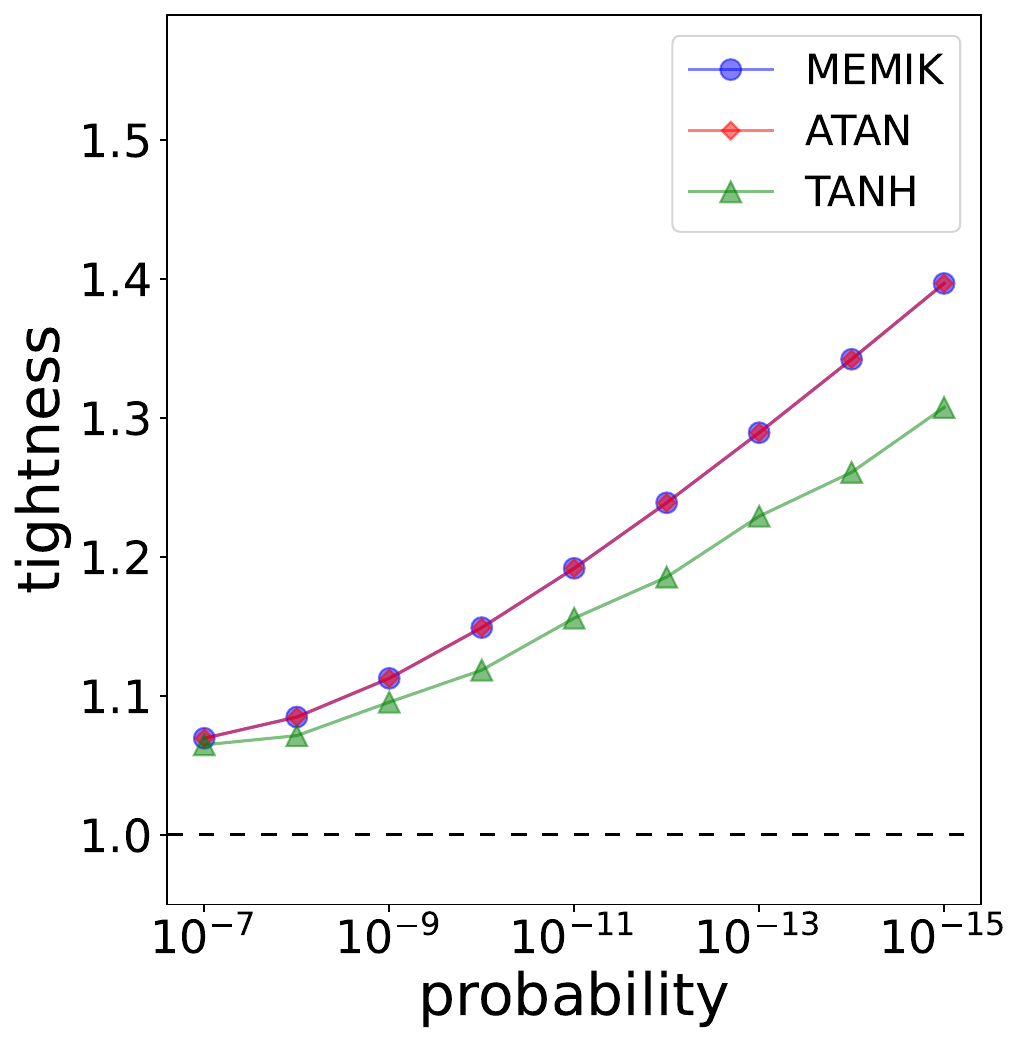}
        \subcaption{Tightness for BetaA}
        \label{fig:eval1_betaA_tightness}
    \end{minipage}
    \caption{Evaluation result for BetaA}
    \label{fig:eval1_betaA}
\vspace{4mm}
\end{figure}
\begin{figure}[tb]
    \centering
    \begin{minipage}[t]{0.45\linewidth}\vspace{0pt}
        \centering
        \includegraphics[width=\linewidth]{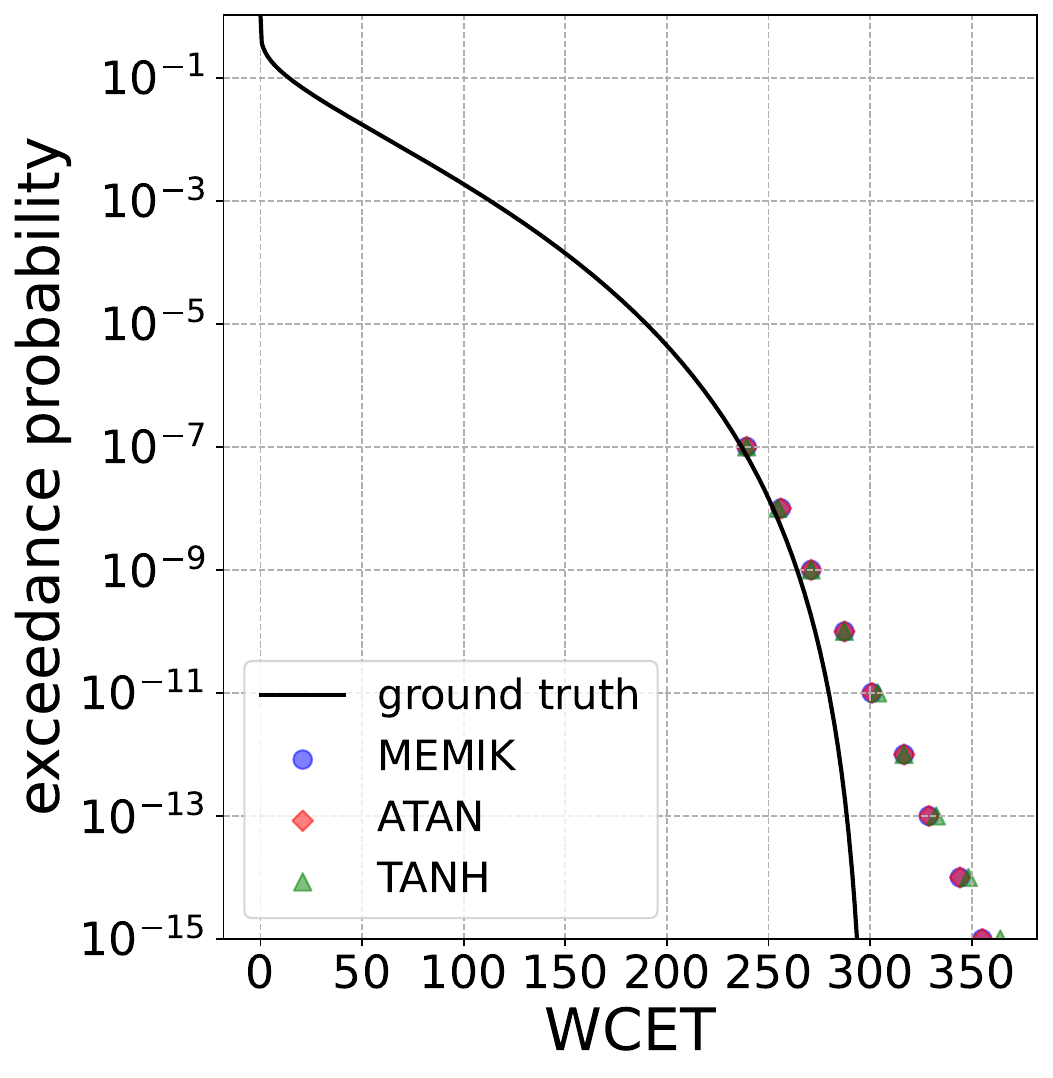}
        \subcaption{Estimated WCET values for BetaB}
        \label{fig:eval1_betaB_ccdf}
    \end{minipage}\hfill
    \begin{minipage}[t]{0.45\linewidth}\vspace{0pt}
        \centering
        \includegraphics[width=\linewidth]{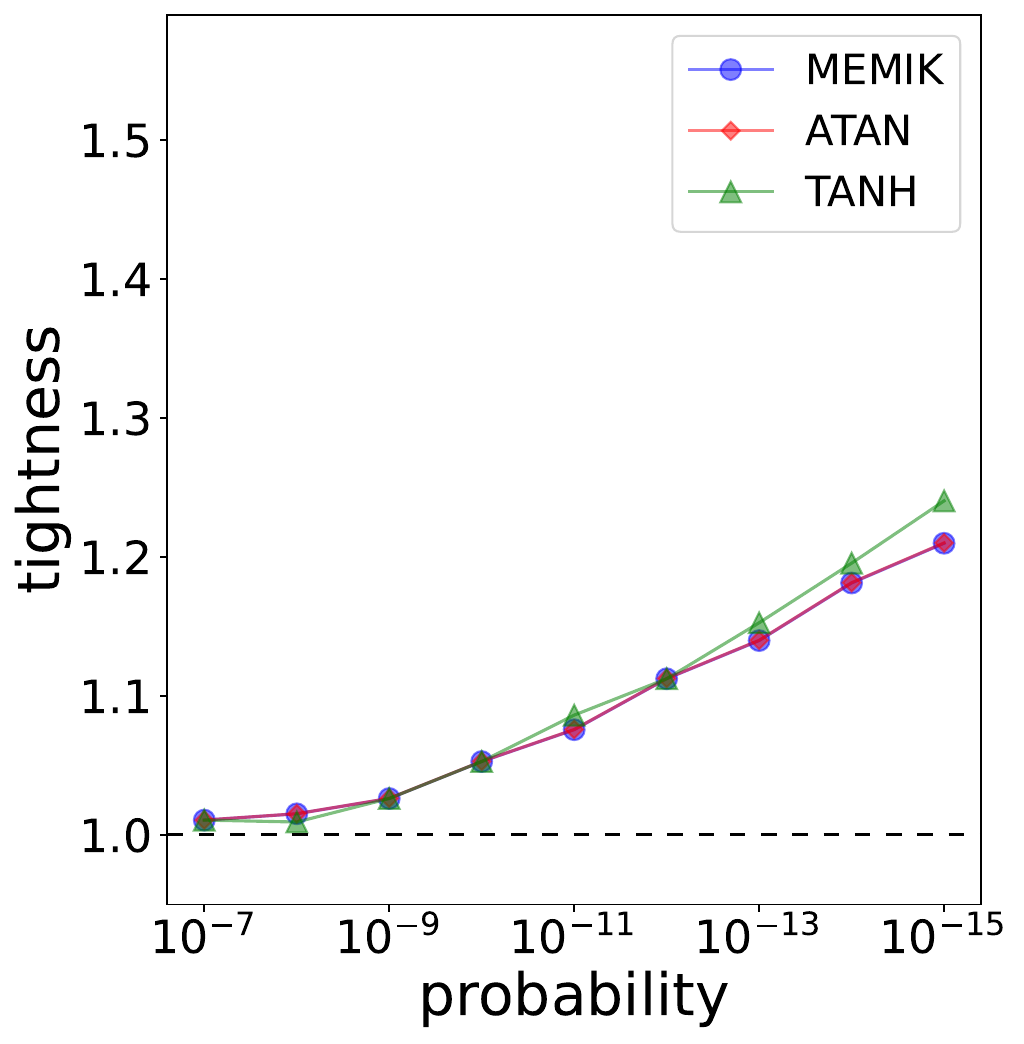}
        \subcaption{Tightness for BetaB}
        \label{fig:eval1_betaB_tightness}
    \end{minipage}
    \caption{Evaluation result for BetaB}
    \label{fig:eval1_betaB}
\vspace{4mm}
\end{figure}
\begin{figure}[tb]
    \centering
    \begin{minipage}[t]{0.45\linewidth}\vspace{0pt}
        \centering
        \includegraphics[width=\linewidth]{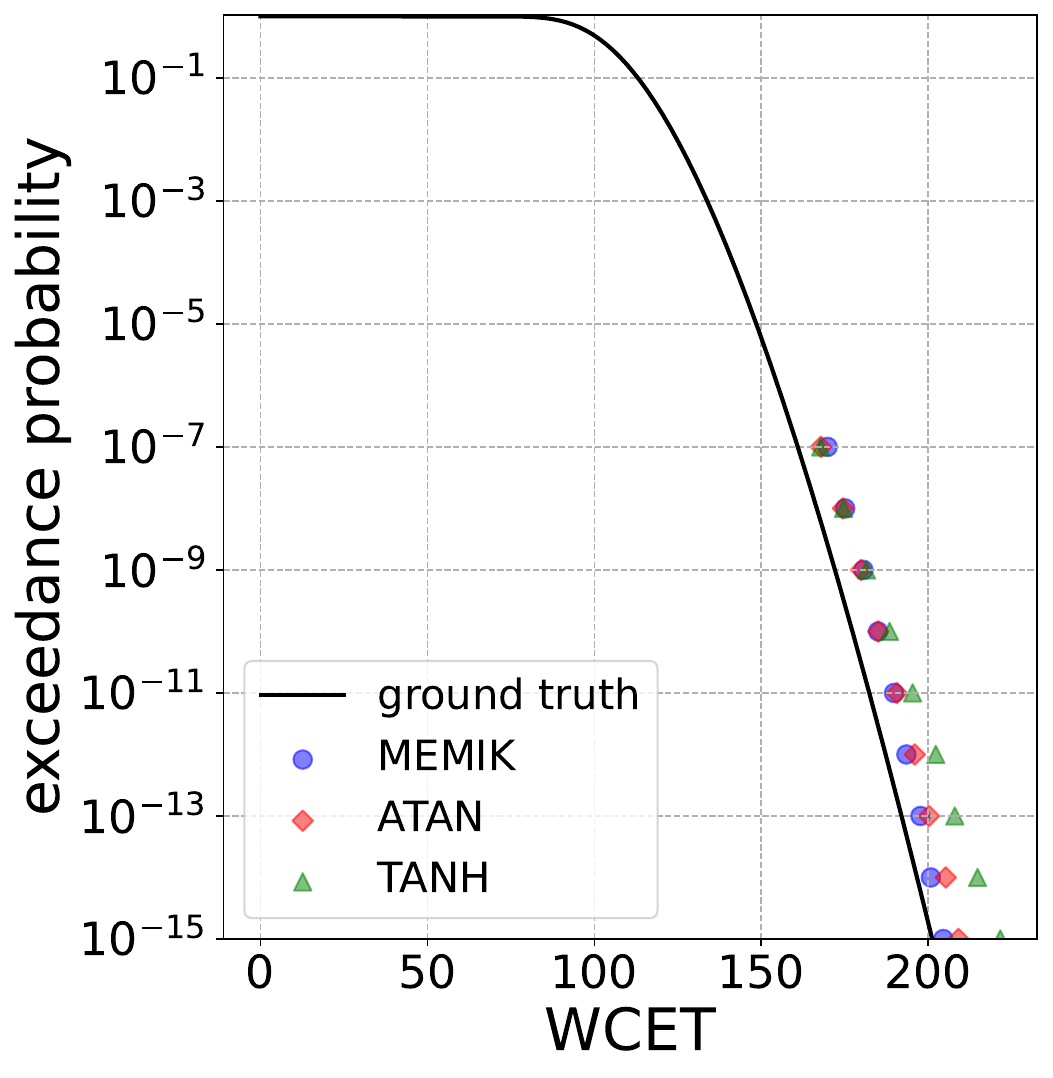}
        \subcaption{Estimated WCET values for GammaA}
        \label{fig:eval1_gammaA_ccdf}
    \end{minipage}\hfill
    \begin{minipage}[t]{0.45\linewidth}\vspace{0pt}
        \centering
        \includegraphics[width=\linewidth]{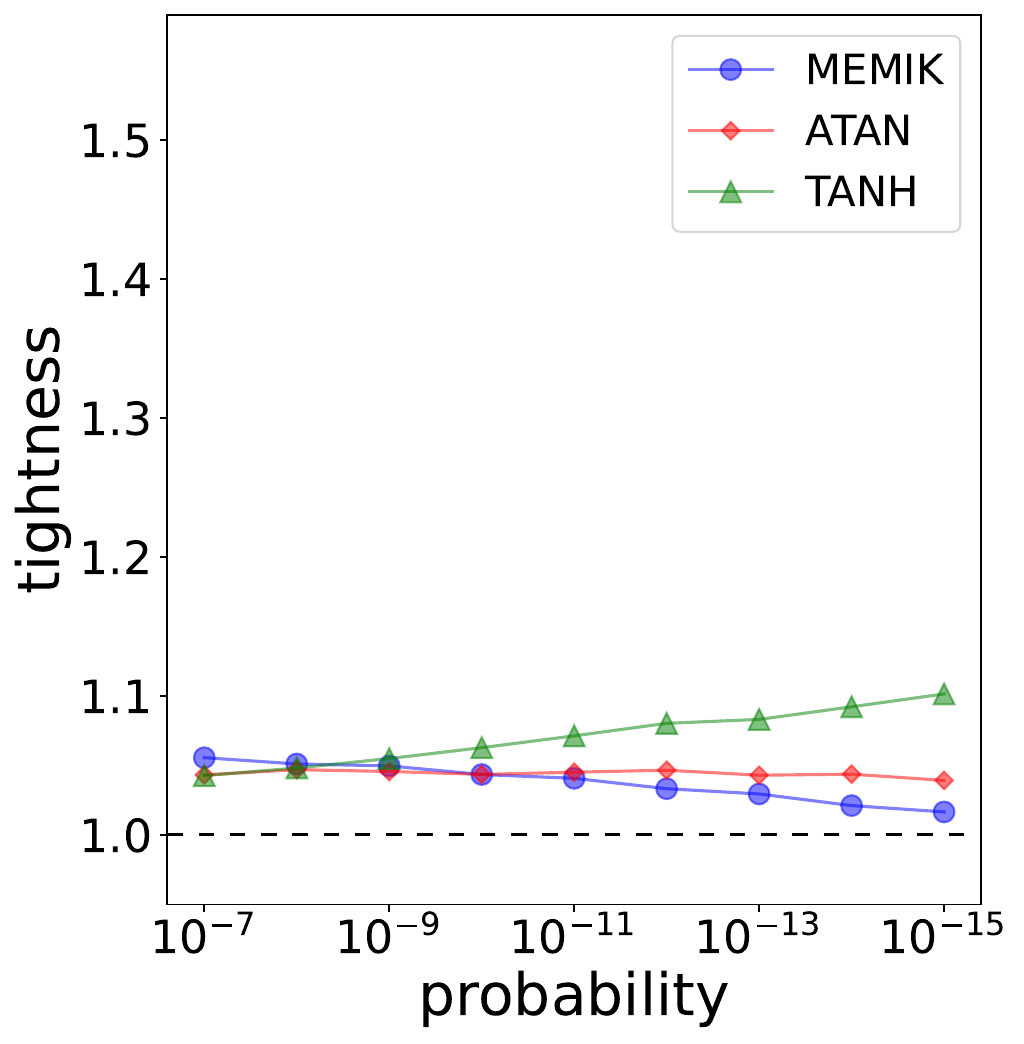}
        \subcaption{Tightness for GammaA}
        \label{fig:eval1_gammaA_tightness}
    \end{minipage}
    \caption{Evaluation result for GammaA}
    \label{fig:eval1_gammaA}
\vspace{4mm}
\end{figure}
\begin{figure}[tb]
    \centering
    \begin{minipage}[t]{0.45\linewidth}\vspace{0pt}
        \centering
        \includegraphics[width=\linewidth]{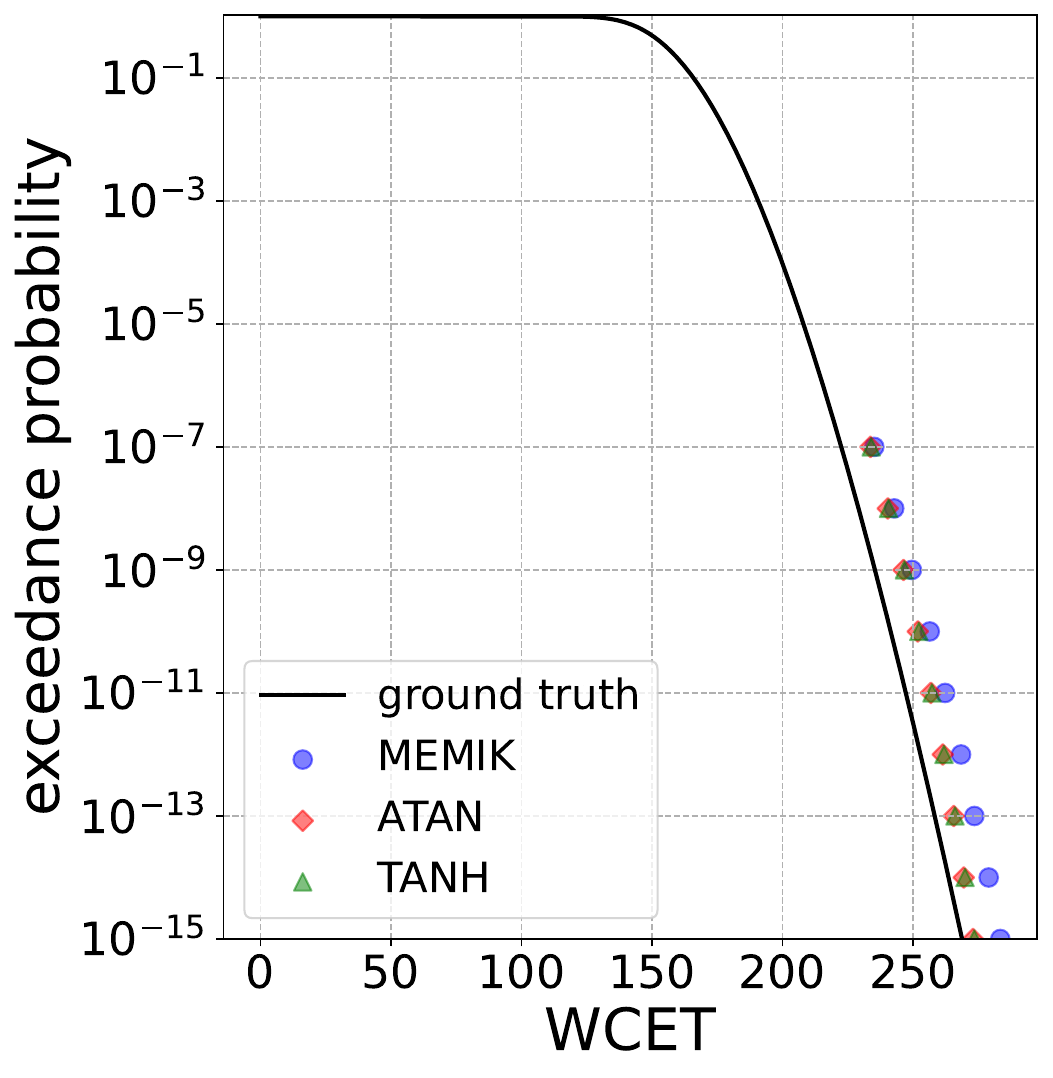}
        \subcaption{Estimated WCET values for GammaB}
        \label{fig:eval1_gammaB_ccdf}
    \end{minipage}\hfill
    \begin{minipage}[t]{0.45\linewidth}\vspace{0pt}
        \centering
        \includegraphics[width=\linewidth]{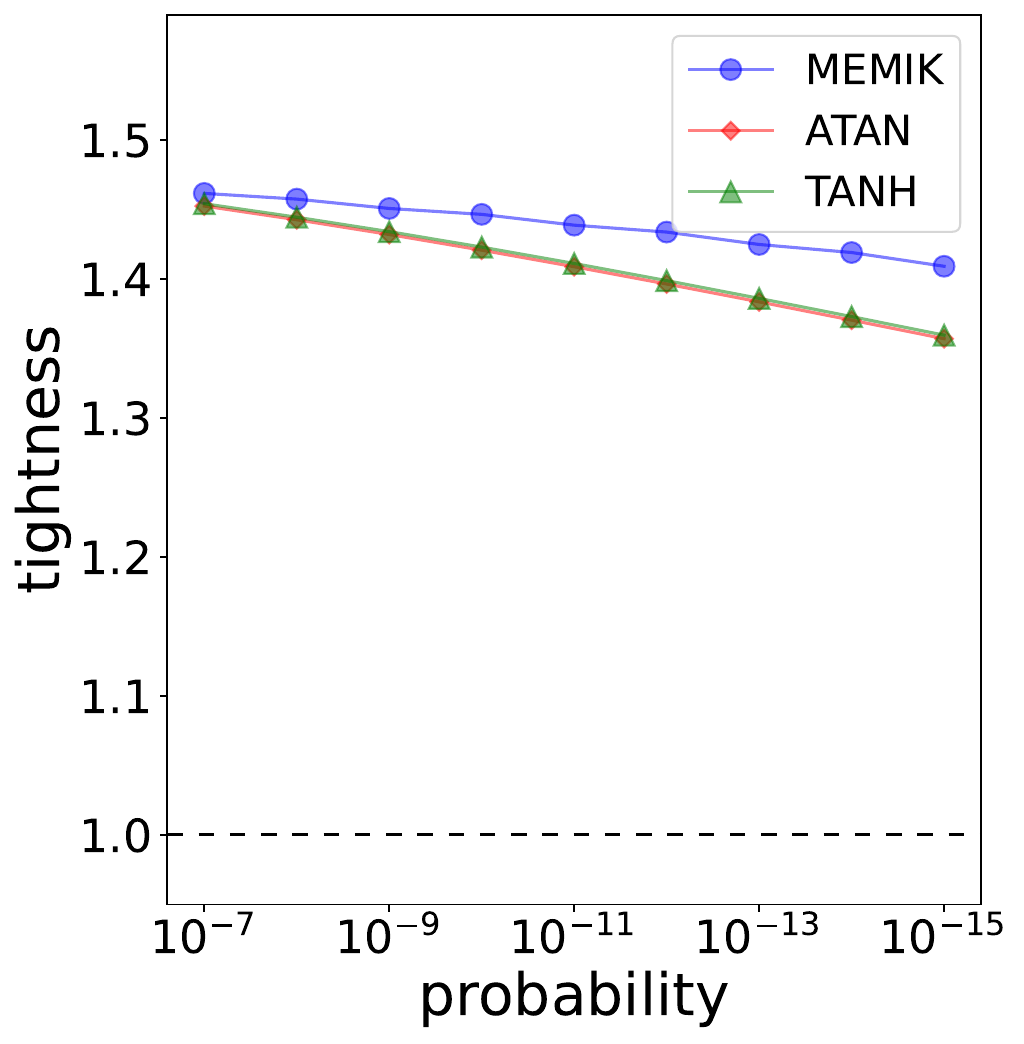}
        \subcaption{Tightness for GammaB}
        \label{fig:eval1_gammaB_tightness}
    \end{minipage}
    \caption{Evaluation result for GammaB}
    \label{fig:eval1_gammaB}
\vspace{4mm}
\end{figure}
\begin{figure}[tb]
    \centering
    \begin{minipage}[t]{0.45\linewidth}\vspace{0pt}
        \centering
        \includegraphics[width=\linewidth]{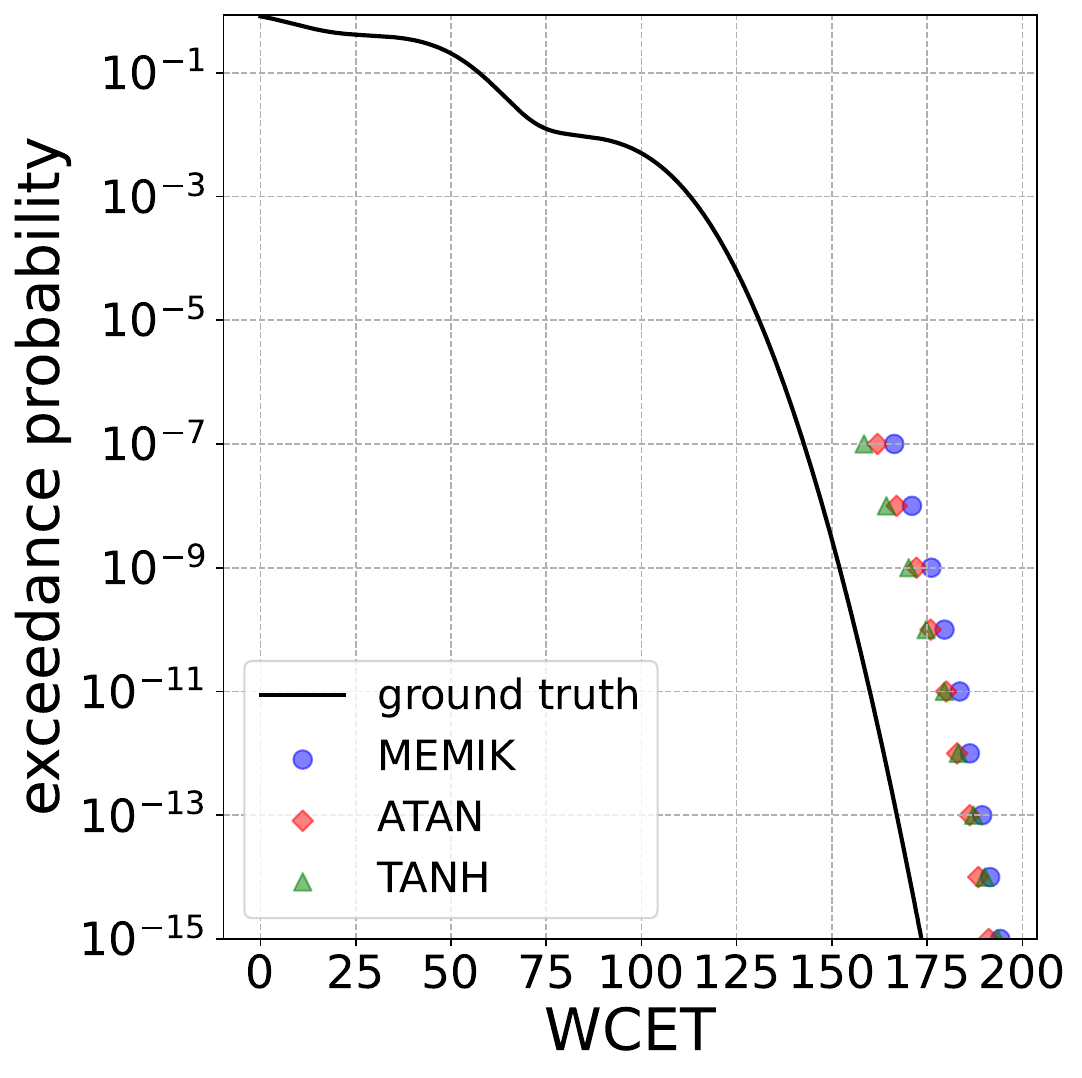}
        \subcaption{Estimated WCET values for MixtureA}
        \label{fig:eval1_mixtureA_ccdf}
    \end{minipage}\hfill
    \begin{minipage}[t]{0.45\linewidth}\vspace{0pt}
        \centering
        \includegraphics[width=\linewidth]{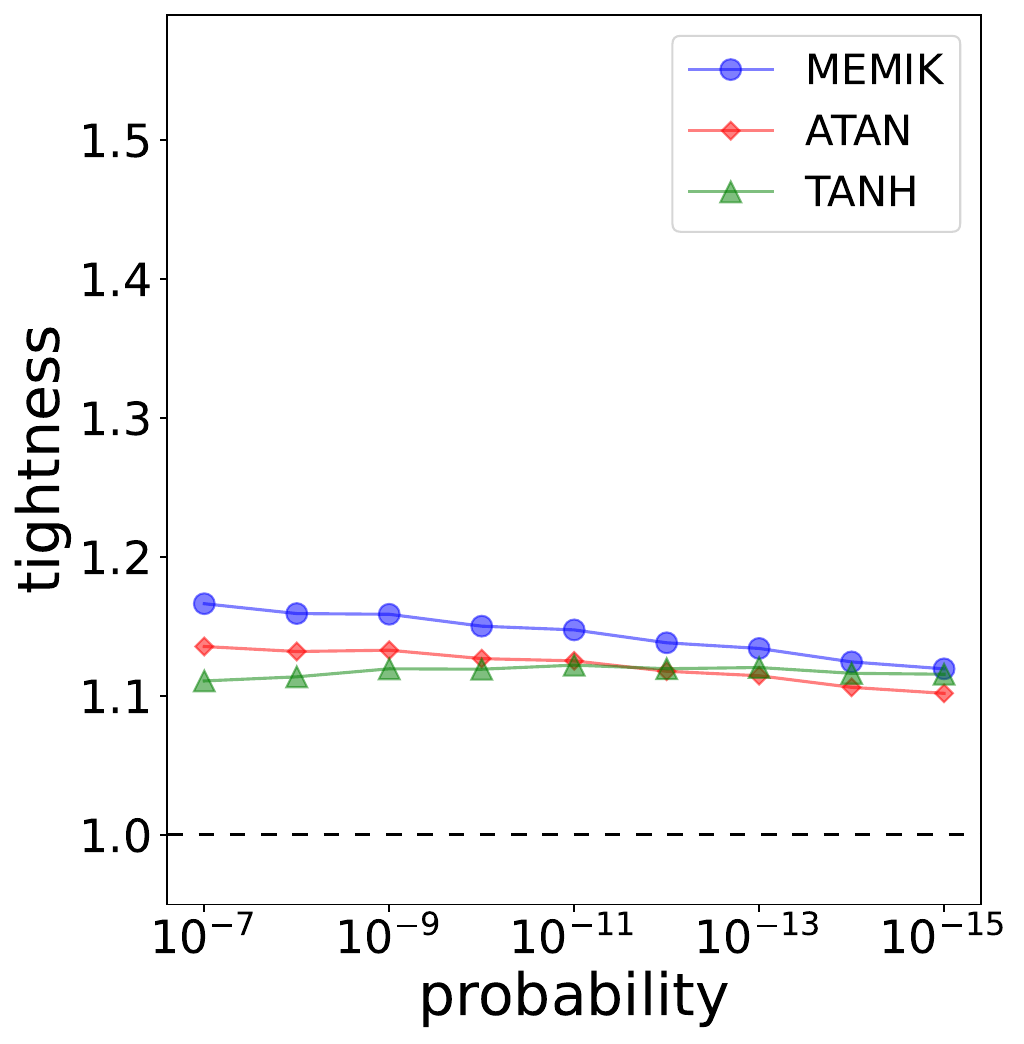}
        \subcaption{Tightness for MixtureA}
        \label{fig:eval1_mixtureA_tightness}
    \end{minipage}
    \caption{Evaluation result for MixtureA}
    \label{fig:eval1_mixtureA}
\vspace{4mm}
\end{figure}
\begin{figure}[tb]
    \centering
    \begin{minipage}[t]{0.45\linewidth}\vspace{0pt}
        \centering
        \includegraphics[width=\linewidth]{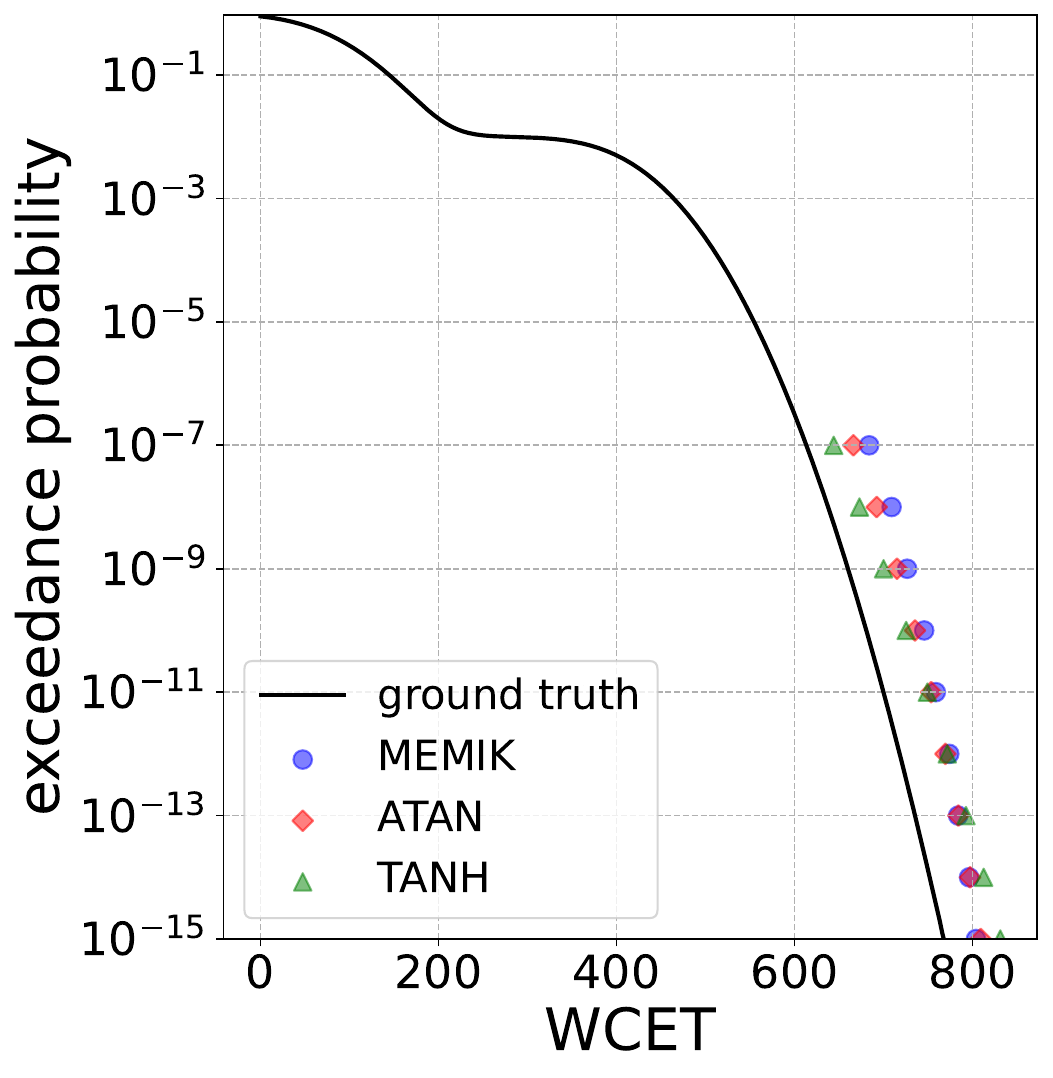}
        \subcaption{Estimated WCET values for MixtureB}
        \label{fig:eval1_mixtureB_ccdf}
    \end{minipage}\hfill
    \begin{minipage}[t]{0.45\linewidth}\vspace{0pt}
        \centering
        \includegraphics[width=\linewidth]{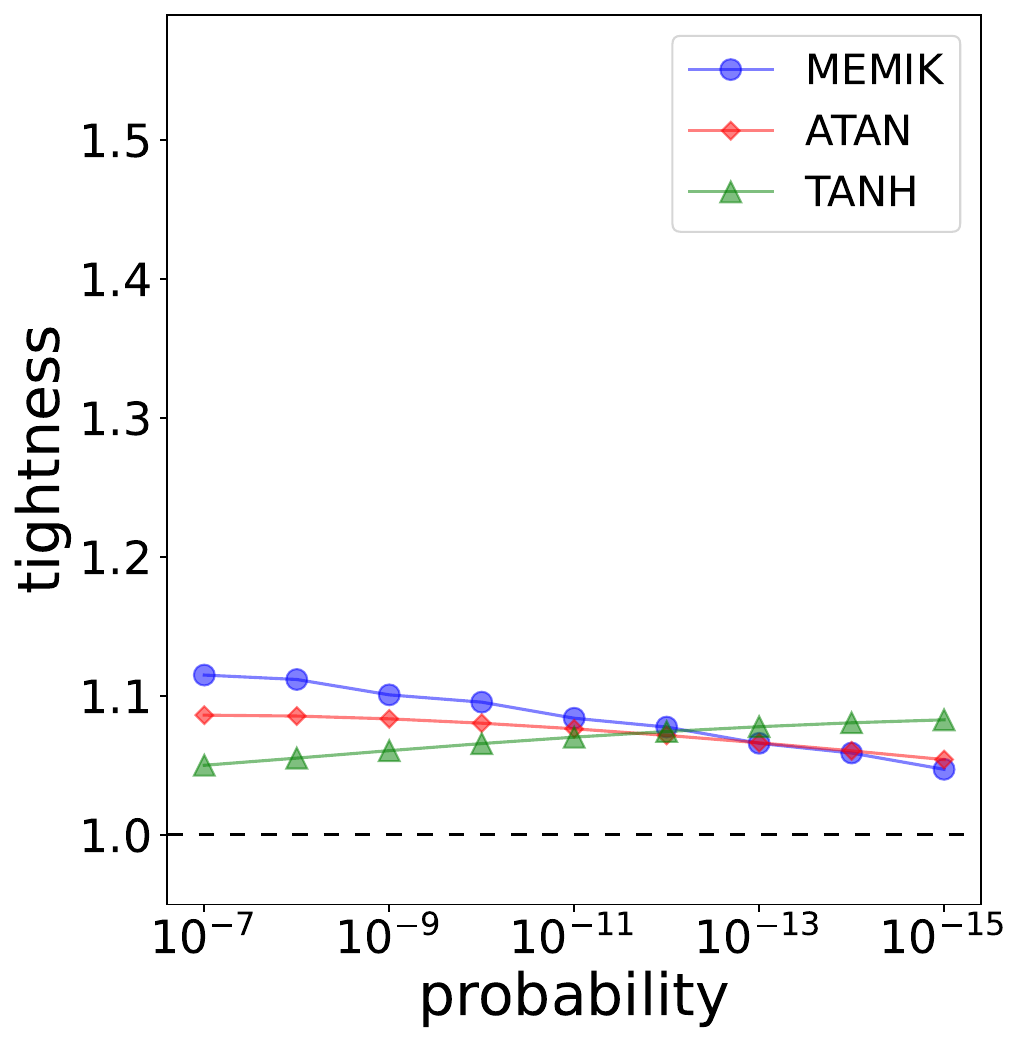}
        \subcaption{Tightness for MixtureB}
        \label{fig:eval1_mixtureB_tightness}
    \end{minipage}
    \caption{Evaluation result for MixtureB}
    \label{fig:eval1_mixtureB}
\vspace{4mm}
\end{figure}
\begin{figure}[tb]
    \centering
    \begin{minipage}[t]{0.45\linewidth}\vspace{0pt}
        \centering
        \includegraphics[width=\linewidth]{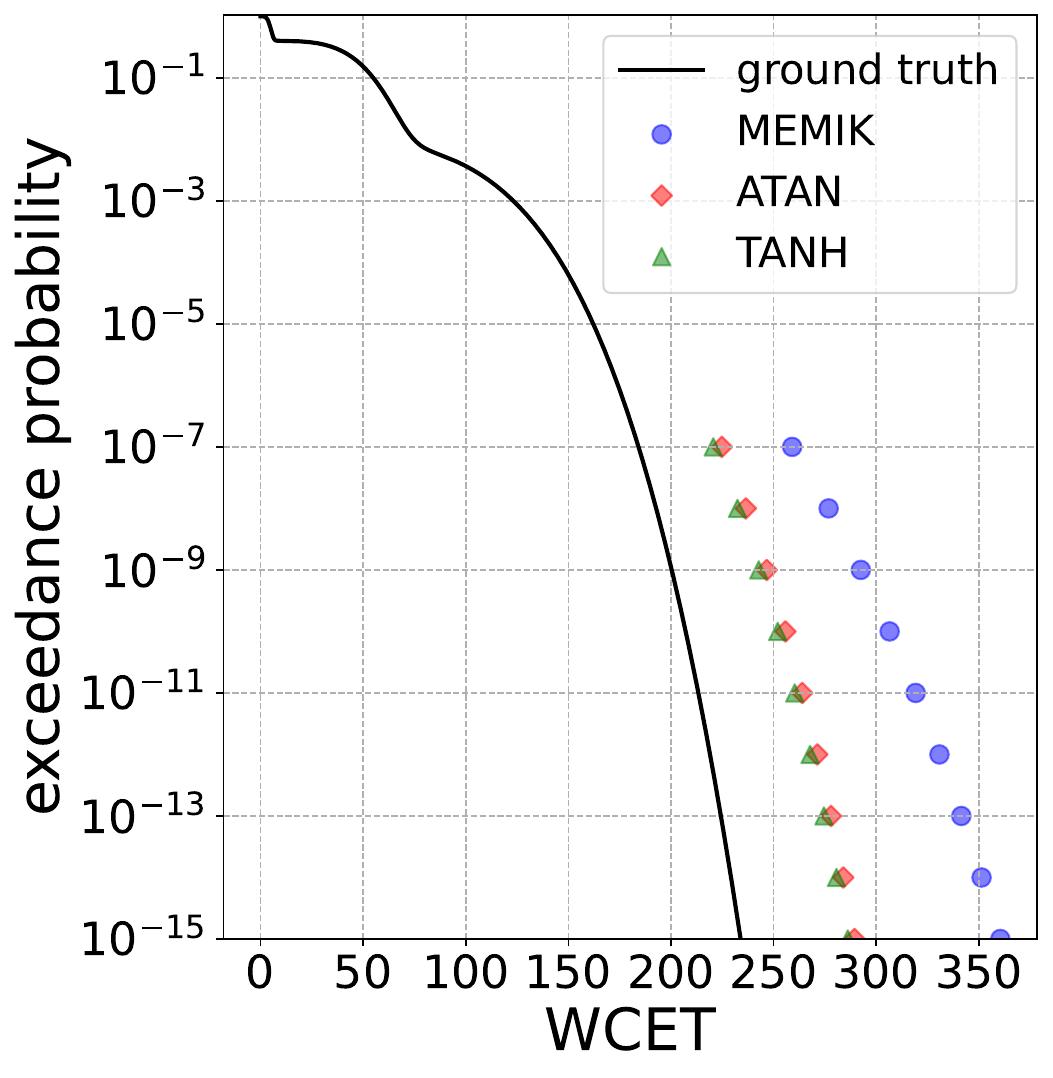}
        \subcaption{Estimated WCET values for MixtureC}
        \label{fig:eval1_mixtureC_ccdf}
    \end{minipage}\hfill
    \begin{minipage}[t]{0.45\linewidth}\vspace{0pt}
        \centering
        \includegraphics[width=\linewidth]{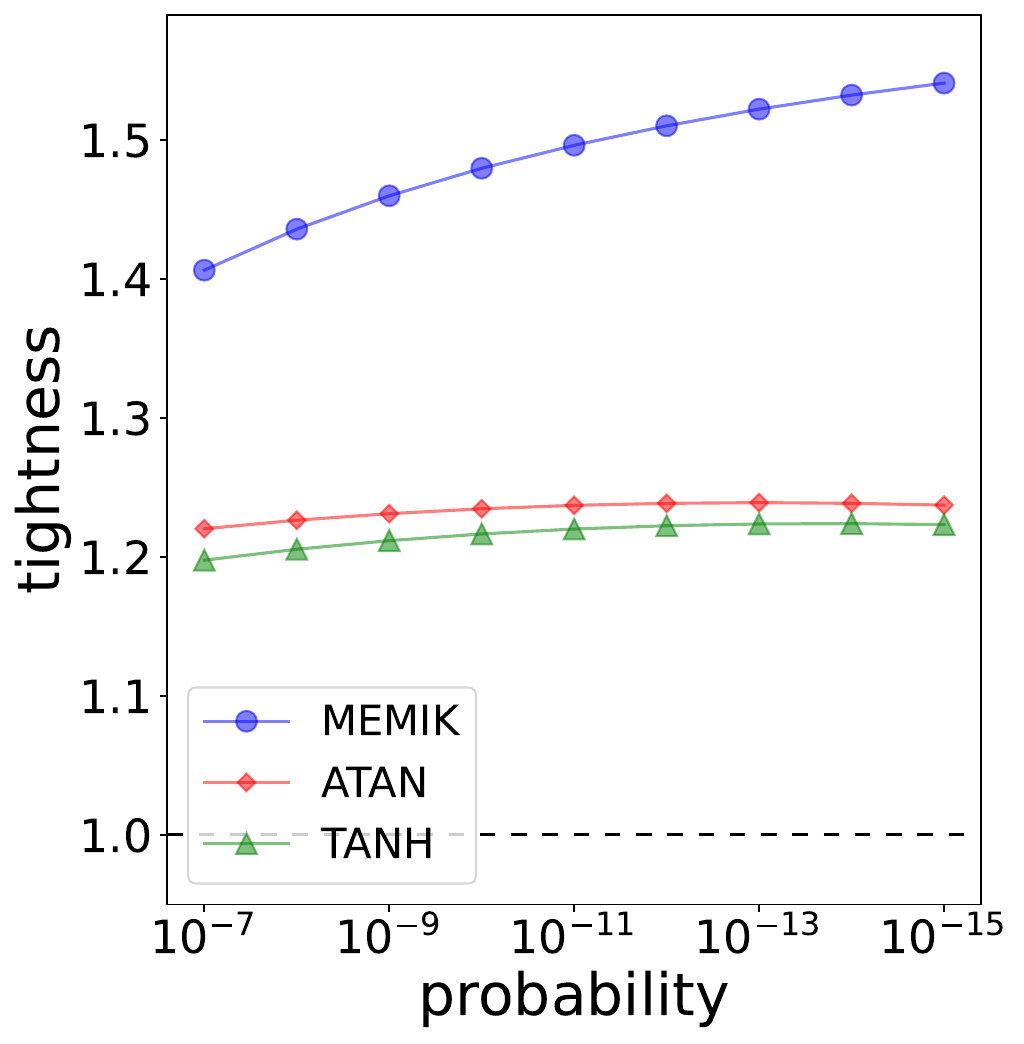}
        \subcaption{Tightness for MixtureC}
        \label{fig:eval1_mixtureC_tightness}
    \end{minipage}
    \caption{Evaluation result for MixtureC}
    \label{fig:eval1_mixtureC}
\vspace{4mm}
\end{figure}
\begin{figure}[tb]
    \centering
    \begin{minipage}[t]{0.45\linewidth}\vspace{0pt}
        \centering
        \includegraphics[width=\linewidth]{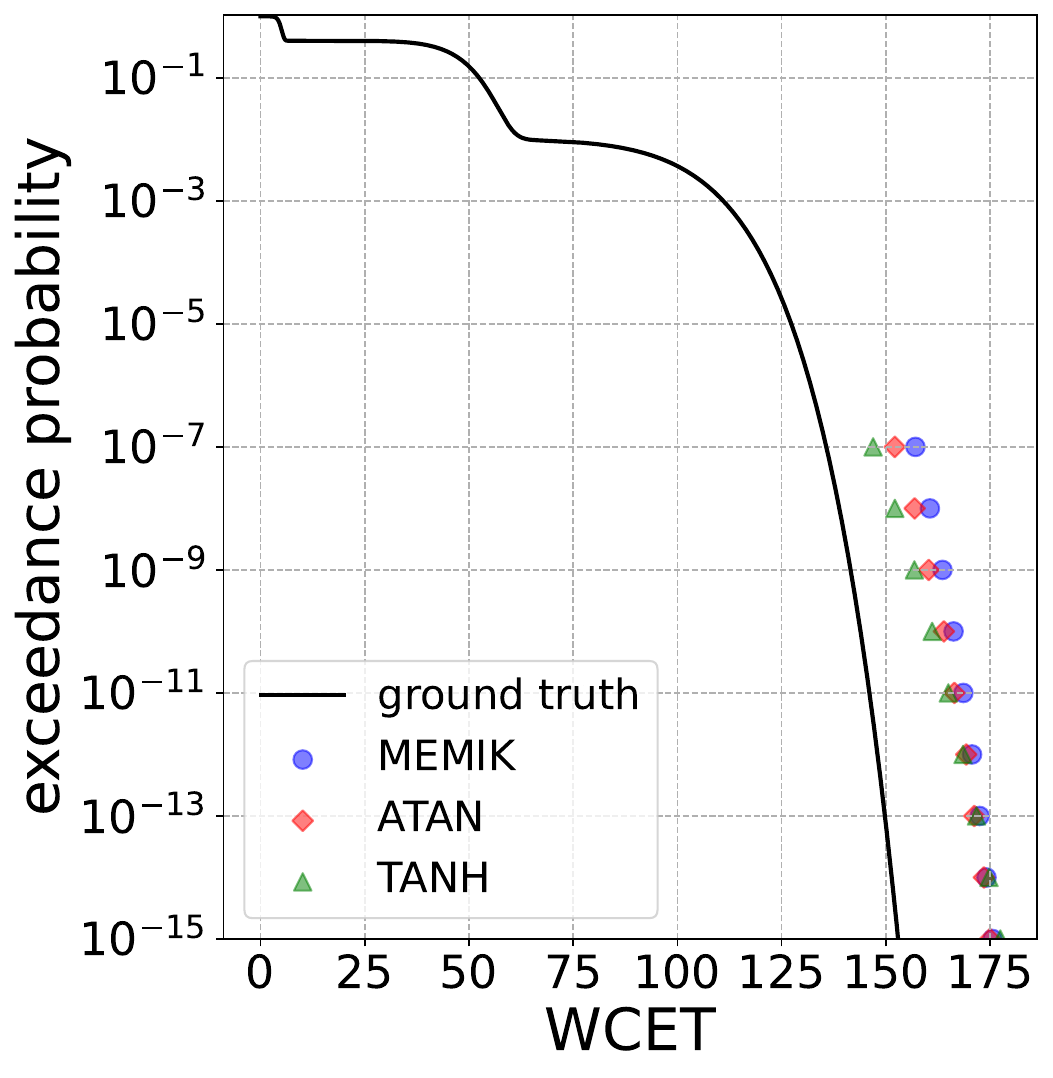}
        \subcaption{Estimated WCET values for MixtureD}
        \label{fig:eval1_mixtureD_ccdf}
    \end{minipage}\hfill
    \begin{minipage}[t]{0.45\linewidth}\vspace{0pt}
        \centering
        \includegraphics[width=\linewidth]{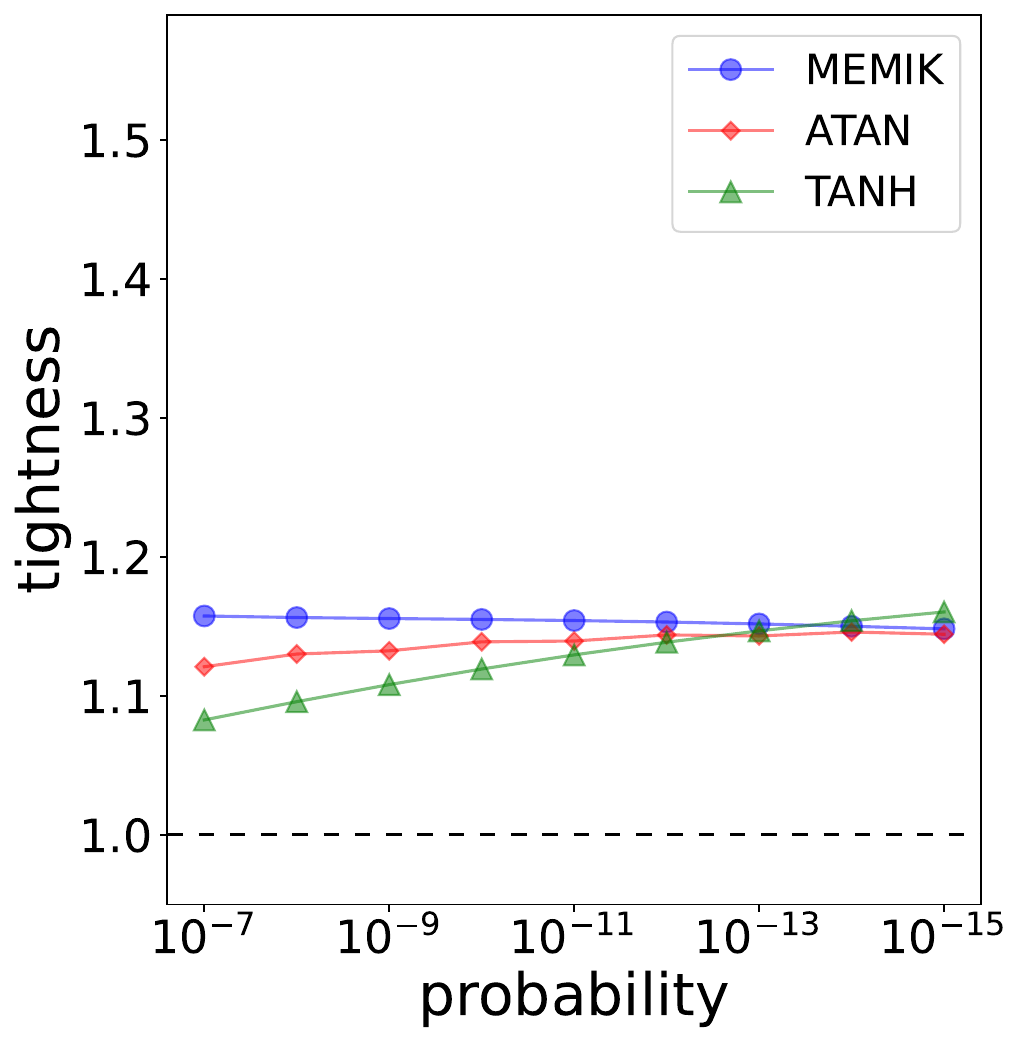}
        \subcaption{Tightness for MixtureD}
        \label{fig:eval1_mixtureD_tightness}
    \end{minipage}
    \caption{Evaluation result for MixtureD}
    \label{fig:eval1_mixtureD}
\end{figure}

Evaluation on the Beta distributions (\figslist{fig:eval1_betaA,fig:eval1_betaB}) shows that tightness grows steeply as the target probability decreases; this rise is more pronounced than for any other distribution.
A closer look at the CCDFs (\figslist{fig:eval1_betaA_ccdf,fig:eval1_betaB_ccdf}) clarifies the cause.
Just at, or slightly beyond, the probability range observable with the $10^{6}$-sample data set, the CCDFs drop off sharply.
Because each estimation procedure relies solely on the observed data and is configured to avoid underestimation, such abrupt tails inevitably inflate the error, so larger discrepancies are expected for distributions of this form.

For the Gamma distributions (\figslist{fig:eval1_gammaA,fig:eval1_gammaB}), the TANH variant performs noticeably worse for GammaA.
This shortcoming is likely due to an inadequate search range for the parameter $d$ in Inequality~\labelcref{eq:memik_tanh}.
Because both saturating functions introduced in the proposed approach contain an almost linear segment, they should not lag far behind MEMIK under suitable parameterization.
If $d$ is chosen too small or too large, however, that guarantee no longer holds, and the range used for GammaA appears sub-optimal.
Another point worth noting is the downward trend in tightness for MEMIK as the target probability decreases; at $p = 10^{-15}$ the value is very close to 1.0, indicating a potential risk of underestimation at even lower probabilities.
In contrast, ATAN maintains relatively stable tightness across the entire probability range.

Gaussian-mixture distributions yield the patterns displayed in \figslist{fig:eval1_mixtureA,fig:eval1_mixtureB}.
In this setting the TANH estimator again behaves less consistently than the other methods, suggesting that the chosen search range for the parameter $d$ was inadequate.
By contrast, MEMIK and ATAN produce very similar estimates, and no sizable gap emerges between them across the evaluated probabilities.

Results for the Weibull-mixture distributions appear in \figslist{fig:eval1_mixtureC,fig:eval1_mixtureD}.
For Mixture C, a clear gap separates MEMIK from the two proposals across the entire probability range.
A sizable improvement under a specific parameter setting was already noted for the single-Weibull case, so the same phenomenon is consistent here.
TANH again departs from both MEMIK and ATAN for Mixture D; this behavior may be mitigated by choosing a more suitable search range for the parameter $d$.

Across all evaluated distributions, no method produced underestimation, and the proposed approaches generally delivered the best tightness.
Although TANH excelled for certain distributions, its performance was not as consistent as that of ATAN.
From the standpoint of versatility, ATAN is preferable, since it consistently outperforms MEMIK while maintaining stable behavior over the full distribution set.

\subsection{Autonomous Driving System Use Case}\label{ssub:eval_actual}

\begin{table}[tb]
    \centering
    \caption{Tightness of the estimates at $p=10^{-5}$ for the Autoware case study}
    \label{tab:eval2_tightness}
    \renewcommand{\arraystretch}{1.1}
    \begin{tabular}{l|ccc}\hline\hline
         & \textbf{MEMIK} & \textbf{ATAN} & \textbf{TANH} \\ \hline
        cbA & 1.21 & 1.20 & 1.18 \\
        cbB & 1.26 & 1.26 & 1.26 \\
        cbC & 1.32 & 1.26 & 1.29 \\
        cbD & 1.12 & 1.12 & 1.09 \\
        cbE & 1.43 & 1.37 & 1.37 \\ \hline
    \end{tabular}
    \vspace{-5mm}
\end{table}

The proposed approach with execution-time traces collected from a real software stack is assessed in this section.
Autoware~\cite{autoware}, an open-source autonomous driving platform, is executed on the AWSIM digital-twin simulator~\cite{awsim}, and five callbacks (cbA -- cbE) are randomly selected for analysis.
A total of $10^{6}$ execution-time observations is recorded.
Because a closed-form ground-truth quantile cannot be obtained from the empirical data, the sample size for each method is fixed at $n = 10^{4}$, and the target exceedance probability is set to $p = 10^{-5}$.
With the available $10^{6}$ observations, this setting yields a quantile that is close to the true value with reasonably high confidence.
Targeting lower probabilities would be desirable, yet collecting more than $10^{6}$ samples already required roughly two weeks of simulation, making a deeper tail study impractical within the given time budget.
As in the synthetic-data experiments, tightness is used as the evaluation metric, ensuring results remain comparable across all data sets.

The results are summarized in \tabref{tab:eval2_tightness}.
For every callback, the proposed methods yield bounds that are either tighter than or equivalent to those obtained with the prior method.
The pattern nevertheless differs from one callback to the next; an outcome that again appears to stem from differences in the underlying execution-time distributions, as already noted in the synthetic-data study.
Each callback clearly exhibits its own distributional characteristics.
With respect to the two variants, ATAN and TANH, no firm conclusion can be drawn here: the real-system trace provides only a single target probability, in contrast to the broad range examined with synthetic data, leaving the comparative evidence inconclusive.

\subsection{Lessons Learned}\label{ssub:eval_lessons}

\textbf{Summary of findings.}
Across the entire experimental evaluation, the two proposed variants behave differently.
ATAN consistently delivers safe and tight bounds for every synthetic and real-system workload, making it the most generic choice.
Although TANH outperforms ATAN (and MEMIK) on a subset of distributions, its advantage is conditional; in several cases it proves less stable than ATAN.

\textbf{Balancing safety and pessimism.}
Introducing saturating functions into Chebyshev's inequality bounds reduces the pessimism while preserving the no-underestimation guarantee.
The empirical results confirm that this trade-off can be achieved without incurring model uncertainty or elaborate hyper-parameter tuning.

\textbf{Impact of the scale parameter $d$.}
The scaling factor $d$ has a decisive influence on tightness.
If $d$ is chosen too large, the bound remains nearly linear and little benefit is gained; if $d$ is too small, the curve saturates too early and tightness deteriorates.
The current study relies on a coarse grid search; although adequate for demonstrating the concept, such exhaustive scans still leave room for sub-optimal choices.
Future work should investigate automated selection strategies, for example, data-driven heuristics or lightweight optimization procedures, that adapt $d$ to the observed sample without exhaustive exploration and with provable safeguards against underestimation.

\section{Related Work}\label{sec:related_work}

Due to the increasing complexity of platforms resulting from the growing demands for computational performance, probabilistic and statistical methods have been gaining more attention as promising solutions.

EVT has emerged as a means to statistically model the WCET of programs.
Since WCET is generally a rare event in actual program executions, EVT, which models the probability of extreme events, is considered suitable for modeling WCET.
A number of measurement-based probabilistic timing analysis (MBPTA) methods are founded on EVT~\cite{2019_ACMSurv}.

EVT-based pWCET estimation can severely impact the reliability of the results obtained if proper verification using statistical tools is not performed.
There is research that introduces statistical tools for verifying each requirement necessary when using EVT and then actually validates executions on commercial multicore platforms~\cite{2023_DAES}.
This study shows that the requirements of EVT are often not met on the target platforms, highlighting the risks and vulnerabilities when applying EVT for pWCET estimation.

Even when the prerequisites of EVT are satisfied, the selection of parameters for the fitted distribution (GEV or GPD) and the input to the EVT process (i.e., sample selection) still have an impact.
Regarding the former, there exists research that quantifies uncertainty by exploring the parameter space based on statistical testing methods~\cite{2020_TECS}.
Furthermore, indicators to evaluate the trade-off between safety and tightness among acceptable distributions are defined, and a methodology to address this is provided.
Unfortunately, concerning model uncertainty due to input, since an optimal sample selection method is not currently provided, it cannot be eliminated or quantified.

Research has been conducted to eliminate model uncertainty in estimation by setting an upper bound on exceedance probability using concentration inequalities that involve higher-order moments~\cite{2022_ECRTS}.
This approach is closely related to our work.
The authors point out that while Markov's inequality is useful as an upper bound without model uncertainty, it yields pessimistic results for high quantiles.
They address this by using power-of-$k$ function.
Their inequality provides a highly accurate upper bound, but they exclude heavy-tailed distributions from consideration. 
Although they claim that the inequality can also be applied to heavy-tailed distributions, applying the inequality to them in practice yields pessimistic results, as explained in \cref{ssub:motivating_example}.

\newcolumntype{Y}{>{\centering\arraybackslash}X}
\begin{table}[tb]
    \centering
    \caption{Comparison of this paper with existing studies} 
    \label{tab:related_work}
    \setlength{\tabcolsep}{3pt}
    \renewcommand{\arraystretch}{1.1}
    \begin{tabularx}{\linewidth}{l|YYYY}\hline\hline
              & addressing statistical uncertainty & addressing model uncertainty & light-tailed distributions & heavy-tailed distributions \\ \hline
    ACM TECS 2020~\cite{2020_TECS} & \checkmark &   & \checkmark & \checkmark \\
    MEMIK with RESTK~\cite{2022_ECRTS} & \checkmark & \checkmark & \checkmark &   \\
    EVT + Copula~\cite{EVT_Copula} &  & \checkmark & \checkmark & \checkmark \\
    This paper & \checkmark & \checkmark & \checkmark & \checkmark \\ \hline
    \end{tabularx}
    \vspace{-3mm}
\end{table}

\section{Conclusion}\label{sec:conclusion}

This paper proposed two new variants of inequality-based pWCET estimation that incorporate arctangent and hyperbolic tangent functions to reduce pessimism while preserving safety, without relying on model assumptions. 
These methods address a key limitation of existing approaches by mitigating the influence of extreme execution times, particularly in heavy-tailed distributions. 
Through extensive experiments using both synthetic and real-world execution-time data, we confirmed that the proposed variants consistently achieve tighter, yet safe, upper bounds.

Despite these promising results, several open issues remain.
Most notably, the quality of the input samples warrants further attention.
As shown in the evaluation, certain combinations of distribution shape and sample size can amplify estimation error.
In addition, the representativeness of the collected traces themselves must be guaranteed.
Determining how to obtain ``good'' samples will therefore be a key objective for future work.

\begin{acknowledgment}
This work was supported by JST AIP Acceleration Research JPMJCR25U1 and JST CREST Grant Number JPMJCR23M1, Japan.
\end{acknowledgment}

\bibliographystyle{ipsjunsrt-e}
\bibliography{reference}

\end{document}